%% file: pcit.tex
\documentclass[oneside,english]{article}

\input{macros}

\makeatletter

\floatstyle{ruled}
\newfloat{algorithm}{tbp}{loa}
\providecommand{\algorithmname}{Algorithm}
\floatname{algorithm}{\protect\algorithmname}

\makeatother

\usepackage{babel}

\begin{document}

\title{Predictive Independence Testing,\\
Predictive Conditional Independence Testing,\\
and Predictive Graphical Modelling}

\author{Samuel Burkart\thanks{\url{samuel.burkart.16@ucl.ac.uk}}
}

\author{
Franz J.~Kir\'{a}ly\thanks{\url{f.kiraly@ucl.ac.uk}}
}

\affil{
Department of Statistical Science,
University College London,\newline
Gower Street,
London WC1E 6BT, United Kingdom
}

\thispagestyle{empty}
\maketitle

%\newpage
\begin{abstract}
Testing (conditional) independence of multivariate random variables is a task central to statistical inference and modelling in general - though unfortunately one for which to date there does not exist a practicable workflow. State-of-art workflows suffer from the need for heuristic or subjective manual choices, high computational complexity, or strong parametric assumptions.\\

We address these problems by establishing a theoretical link between multivariate/conditional independence testing, and model comparison in the multivariate predictive modelling aka supervised learning task. This link allows advances in the extensively studied supervised learning workflow to be directly transferred to independence testing workflows - including automated tuning of machine learning type which addresses the need for a heuristic choice, the ability to quantitatively trade-off computational demand with accuracy, and the modern black-box philosophy for checking and interfacing.\\

As a practical implementation of this link between the two workflows, we present a python package 'pcit', which implements our novel multivariate and conditional independence tests, interfacing the supervised learning API of the scikit-learn package. Theory and package also allow for straightforward independence test based learning of graphical model structure.\\

We empirically show that our proposed predictive independence test outperform or are on par to current practice, and the derived graphical model structure learning algorithms asymptotically recover the 'true' graph. This paper, and the 'pcit' package accompanying it, thus provide powerful, scalable, generalizable, and easy-to-use methods for multivariate and conditional independence testing, as well as for graphical model structure learning.
\end{abstract}
\newpage
\tableofcontents{}

\newpage{}

\newpage

\include{1_introduction}
\include{2_background}
\include{3_pcit-theory}
\include{4_graphical-models}
\include{5_algorithms}
\include{6_workflow}

\include{7_experiments}
\include{8_conclusions}
\include{9_appendix}

\bibliographystyle{plainnat}
\bibliography{pcit}

\end{document}

%% file: macros.tex
\usepackage[utf8]{inputenc}

\usepackage{hyperref}
\usepackage{url}

\usepackage{algpseudocode,algorithm,algorithmicx}

\usepackage{mathtools}
\usepackage{amsmath,amssymb,amsthm}
\usepackage{thmtools}

\usepackage{graphicx}

\usepackage{todonotes}

\usepackage[mathcal]{eucal}
\usepackage{bbm}

\usepackage[shortlabels]{enumitem}

\usepackage[margin=1.2 in]{geometry}
\usepackage{pdflscape}

\usepackage{multirow}
\usepackage{array}
\usepackage[numbers]{natbib}
\usepackage{authblk}

\usepackage{titlesec}
\titleformat{\section}
	{\normalfont\Large\bfseries\filcenter}{\thesection.}{1 ex}{}
\titleformat{\subsection}%[runin]
	{\normalfont\large\bfseries}{\thesubsection.}{1 ex}{}
\titleformat{\subsubsection}%[runin]
	{\normalfont\normalsize\bfseries}{\thesubsubsection.}{1 ex}{}

\declaretheorem[name=Theorem]{Thm}
\declaretheorem[within=section,name=Lemma]{Lem}
\declaretheorem[sibling=Lem,name=Definition]{Def}

\declaretheorem[sibling=Lem,name=Notation]{Not}
\declaretheorem[sibling=Lem,name=Proposition]{Prop}
\declaretheorem[sibling=Lem,name=Remark]{Rem}

\declaretheorem[sibling=Lem,name=Corollary]{Cor}

\newcommand{\argmin}{\operatornamewithlimits{argmin}}
\newcommand{\argmax}{\operatornamewithlimits{argmax}}

\newcommand{\Var}{\operatorname{Var}}

\newcommand{\card}{\#}

\newcommand{\calD}{\mathcal{D}}

\newcommand{\calG}{\mathcal{G}}

\newcommand{\calL}{\mathcal{L}}

\newcommand{\calT}{\mathcal{T}}

\newcommand{\calX}{\mathcal{X}}
\newcommand{\calY}{\mathcal{Y}}
\newcommand{\calZ}{\mathcal{Z}}

\newcommand{\RR}{\ensuremath{\mathbb{R}}}

\newcommand{\EE}{\ensuremath{\mathbb{E}}}

\makeatletter
\providecommand*{\diff}%
        {\@ifnextchar^{\DIfF}{\DIfF^{}}}
\def\DIfF^#1{%
        \mathop{\mathrm{\mathstrut d}}%
                \nolimits^{#1}\gobblespace
}
\def\gobblespace{%
        \futurelet\diffarg\opspace}
\def\opspace{%
        \let\DiffSpace\!%
        \ifx\diffarg(%
                \let\DiffSpace\relax
        \else
                \ifx\diffarg\[%
                        \let\DiffSpace\relax
                \else
                        \ifx\diffarg\{%
                                \let\DiffSpace\relax
                        \fi\fi\fi\DiffSpace}
\makeatother

\newcommand{\indep}{\rotatebox[origin=c]{90}{$\models$}} % Independent symbol

\usepackage{tcolorbox}

\usepackage{subcaption}

\usepackage{tikz}
\usetikzlibrary{bayesnet}

\newcommand*\Let[2]{\State #1 $\gets$ #2}

\usepackage{listings}

\usepackage{xcolor} %custom colours

\definecolor{lightgray}{gray}{0.9}

\makeatletter
\renewcommand\paragraph{\@startsection{paragraph}{4}{\z@}%
                                    {1ex \@plus1ex \@minus.2ex}%
                                    {-1em}%
                                    {\normalfont\normalsize\bfseries}}

\makeatother

%% file: 1_introduction.tex
\section{Introduction}

\subsection{Setting: testing independence}

The study of dependence is at the heart of any type of statistical analysis, and independence testing is an important step in many scientific investigations, be it to determine if two things are related, or to assess if an intervention had the desired effect, such as:

\begin{itemize}
\item When conducting market research, one might be interested in questions such as ``Are our advertisement expenditures independent of our profits?'' (hopefully not), or the more sophisticated version ``conditional on the state of the market, are advertisement expenditures independent of our profits?'', which, if found to be true, would mean we are unlikely to increase profits through an increase in our advertising budget (subject to the usual issue that inferring causality from data requires an intervention or instrument).
\item When collecting data for a medical study on the occurrence of an outcome $Y$, one might ask ``In the presence of data about \textit{attributes A} for the subjects, should we still collect data for \textit{attributes B}''. If Y is independent of \textit{attributes B} given \textit{attributes A}, additionally collecting information about \textit{attributes B} will not improve the knowledge of the state of $Y$ (subject to the usual issue that this conclusion is valid only for patients sampled in the same way from the same population).
\end{itemize}

The difficulty of the independence testing task crucially relies on whether the following two complications present:\\
\begin{itemize}
\item {\bf Multivariate independence testing.} This concerns the type of values which the involved variables, in the second example the values which \textit{attributes A} and \textit{attributes B} may take: if the domain of possible values consists either of a single number (continuous variable), or one class out of many (categorical variable), the hypothesis test is ``univariate'', and powerful methodology exists that deals well with most scenarios, subject to some constraints. Otherwise, we are in the ``multivariate hypothesis testing'' setting.
\item {\bf Conditional independence testing.} Whether there are conditioning random variables which are to be controlled for, in the sense of testing independence conditional on a possible third \textit{attributes C}. If so, we are in the ``conditional hypothesis testing'' setting.
\end{itemize}
For the task which is neither multivariate nor conditional, well-recognized and universally applicable hypothesis tests (such as the t-test or chi-squared test) are classically known.
The multivariate setting and the conditional setting are less well studied, and are lacking approaches which are general and universally accepted, due to difficulties in finding a suitable approach which comes with theoretical guarantees and is free of strong model assumptions. The setting which is both multivariate and conditional is barely studied.
The three major state-of-the-art approaches are density estimation, copula and kernel based methods. Most instances are constrained to specific cases or rely on subjective choices that are difficult to validate on real-world data. A more detailed overview of the state-of-art and background literature is given in Section~\ref{sec:background}.

\subsection{Predictive independence testing}
The methodology outlined in this paper will consider multivariate and conditional independence testing from a new angle. The underlying idea for the test is that if two random variables $X,Y$ are independent, it is impossible to predict $Y$ from $X$ - in fact it will be shown that these two properties are equivalent (in a certain quantitative sense). The same applies to conditional independence tests: two random variables $X,Y$ are conditionally independent given $Z$, if adding $X$ as predictor variable above $Z$ will not improve the prediction of $Y$. In both cases, the predictive hypothesis test takes the form of comparing a good prediction strategy against an optimal baseline strategy via a predictive loss function. By determining if losses stemming from two predictions are significantly different, one can then test statistically if a variable adds to the prediction of another variable (potentially in the presence of a conditioning set), which by the equivalence is a test for (conditional) independence.

\subsection{Graphical model structure learning}
Probabilistic graphical models are a concept that heavily relies on independence statements for learning and inference. Most structure learning algorithms to date are, as a result of the lack of a scalable conditional independence tests and additional combinatorial issues, constraint-based, or make heavy assumptions on underlying distributions of a sample. This paper will leverage the predictive independence test into a new routine to estimate the undirected graph for the distribution underlying a sample, based on conditional independence testing, allowing it to make only weak assumptions on the underlying distribution.

\subsection{Principal contributions}
The new approach to multivariate and conditional independence testing outlined in this paper improves concurrent methodology by deriving an algorithm that

\begin{itemize}
    \item features a principled model selection algorithm for independence testing by linking the field of independence testing to the field of predictive modelling, thus filling a gap in state-of-the-art methodology,
    \item additionally allowing independence testing to directly benefit from the well-understood and efficiently implemented theory for model selection and parameter tuning in predictive modelling,
    \item is comparatively fast and scalable on a wide variety of problems and
    \item deals with the multivariate and conditional independence testing task in a straightforward manner.
\end{itemize}

Additionally, an algorithm leveraging the newly derived test into a scalable independence testing-based graphical model structure learning algorithm is outlined, which overcomes issues in the field by offering a test for undirected graph structure learning that offers stringent methodology to control the number of type 1 errors in the estimated graph.

\subsection{Paper overview}
Section \ref{sec:background} will provide an overview of the most important tasks in and approaches to statistical independence testing and outline the issues in current methodology. Section \ref{sec:depend} will then propose and derive a novel approach to independence testing, the predictive conditional independence test (PCIT). After, section \ref{sec:gm} will introduce the concept of a graphical model and survey the current structure learning algorithms. Section \ref{sec:API} will then state the relevant algorithms of the Python implementation, which is outlined in section \ref{sec:package}. Section \ref{sec:experiments} provides performance statistics and examples for the conditional independence test as well as the graphical model structure learning algorithm. Lastly, section \ref{sec:conclusion} will describe the advantages of using the PCIT for independence testing and outline drawbacks and directions for further research.

\subsection*{Authors' contributions}
This manuscript is based on SB’s MSc thesis, submitted September 2017 at University College London and written under supervision of FK, as well as on an unpublished manuscript of FK which relates predictive model selection to statistical independence. The present manuscript is a substantial re-working of the thesis manuscript, jointly done by SB and FK.\\
FK provided the ideas for the independence tests in Section~\ref{sec:depend} and the usage of them for graphical models, SB and FK jointly conceived the ideas for the graphical model structure learning algorithm. Literature overview is due to SB with helpful pointers by FK. Python package and experiments are written and conducted by SB, partially based on discussions with FK.

\subsection*{Acknowledgements}
We thank Fredrik Hallgren and Harald Oberhauser for helpful discussions.\\
FH has briefly worked on a thesis about the same topic under the supervision of FK, before switching to a different topic. While FH did, to our knowledge, not himself make any of the contributions found in this paper, discussions with him insipred a few of them.\\
HO pointed FK in generic discussions about loss functions towards some prior results (elicitation of median, and elicitation as a concept as defined by~\citet{gneiting2007strictly}) which helped making some of the proven statements more precise (such as the Q-losses being faithful for the univariate real case).
HO also pointed out the work of~\citet{lopez2016revisiting}.\\
We thank Jan Limbeck for the suggestion of the term ``uninformed'' for describing uninformed predictors (in replacement of, e.g., ``naive'' and ``stupid'').\\
We thank~\citet{sen2017model} for making us aware of their recent work.\\
Section~\ref{Sec:badbaseline} has previously appeared in similar form in the manuscript by~\citet{gressmann2018probabilistic} which post-dates the earliest pre-print version of this manuscript, but pre-dates the current one. Due to the frequent use of permutation baselines in contemporary work, it has been copied almost verbatim while adapting it to the non-probabilistic setting.\\

FK acknowledges support by the Alan Turing Institute, under EPSRC grant EP/N510129/1. 

%% file: 2_background.tex
\section{Statistical independence testing}\label{sec:background}

\subsection{About (in)dependence}
Statistical (in)dependence is a property of a set of random variables central to statistical inference.
Intuitively, if a set of random variables are statistically independent, knowing the value of some will not help in inferring the values of any of the others.\\[0.1in]
Mathematically: let $X$, $Y$ and $Z$ be random variables taking values in $\calX$, $\calY$ and $\calZ$ respectively. As usual, we denote for $A\in\calX$ by $P(X \in A)$ the probability of $X$ taking a value in $A$, and for $A\in\calX, C\in\calZ$ by $P(X \in A| Z \in C):=P(X \in A, Z \in C)/P(Z \in C)$ be the conditional probability of $X$ taking a value in $A$ when $Z$ is known/observed to have taken a value in $C$.

\begin{Def}\label{def:margindep}
$X$ and $Y$ are called \textbf{marginally independent} (of each other) if for all $A \subseteq \calX$ and $B \subseteq \calY$ (where the below expression is defined) it holds that
$$P(X \in A, Y \in B) = P(X \in A)P(Y \in B).$$
\end{Def}

This formulation allows for $X$ and $Y$ to be defined over sets of random variables that are a mixture of continuous and discrete, as well as being univariate or multivariate (thus implicitly covering the case of multiple univariate random variables as well).

\begin{Def}\label{def:indep}
$X$ and $Y$ are called \textbf{conditionally independent} (of each other) given $Z$ if for all $A \subseteq \mathcal X, B \subseteq \mathcal Y,$ and $C \subseteq \mathcal Z$  (where the below expression is defined) it holds that
$$P(X \in A, Y \in B | Z \in C) = P(X \in A | Z \in C)P(Y \in B | Z \in C).$$
\end{Def}

For absolutely continuous or discrete $X$ and $Y$, Definition \ref{def:margindep} straightforwardly implies that marginal independence is equivalent to the joint distribution or mass function factorizing, i.e., it equals the product of the marginal distributions' probability or mass function. The analogue result also shows for the conditional case.\\

We would like to note that the mathematical definition independence is symmetric, i.e., having the property from Definition~\ref{def:margindep} is unaffected by interchanging $X$ and $Y$. In contrast, the intuitive motivation we gave at the beginning however, namely that knowing the value of $X$ does not yield any additional restriction regarding the value of $Y$, is non-symmetric. This non-symmetric statement of statistical independence is commonly phrased mathematically as $P(Y \in B| X \in A) = P(Y \in B)$ which is directly implied by Definition~\ref{def:margindep} (and that of the conditional).\\

This non-symmetric characterization of statistical independence is also morally at the heart of the predictive characterization which we later give in the supervised learning setting: $Y$ cannot be predicted from $X$, phrased in a quantitative way that connects checking of statistical independence to estimation of a supervised generalization error difference which eventually allows for a quantitatve testing of the hypothesis of independence.

\subsection{Statistical independence testing}

Testing whether two random variables are (conditionally) independent is one of the most common tasks in practical statistics (perhaps the most common one after summarization) and one of the most important topics in the theoretical study of statistical inference (perhaps the most important one).

The most frequently found types of quantitative certificates for statistical independence are phrased in the (frequentist) Neyman-Pearson hypothesis testing paradigm. In terms of the methodological/mathematical idea, the main distinction of independence tests is into parametric (where the data is assumed to stem from a certain model type and/or distribution) and non-parametric (''distribution-free'') tests.
The hypothesis tested usually takes one of the three following forms:

\subsubsection*{Marginal independence, $X \indep Y$}
Marginal independence of two random variables is the topos of much of classical statistics.
Well known classical test statistics for the continuous univariate case are the Pearson correlation coefficient, Kendall's $\tau$ and Spearman's $\rho$. For discrete variables, Pearson's $\chi$-squared statistic can be used. From theoretical results on these statistics' asymptotic or exact distribution, univariate independence tests may be derived. For $X$ being discrete and $Y$ being continuous, t-family tests or Wilcoxon family tests may be used.

Testing with at least one of $X$ or $Y$ being multivariate is much more difficult and less standard, current methodology will be outlined in the next section.

\subsubsection*{Conditional independence, $X \indep Y | Z$}
Testing for conditional independence is an inherently more difficult problem than marginal independence \citep{bergsma2004}. Hence, few viable options exist to date to test if two sets of variables are conditionally independent given a conditioning set. The most common instances are either based on parametric model assumptions (e.g., linear) or on binning $Z$ and $X$, then comparing if the distribution of $Y$ conditioned on $Z$ changes, when additionally conditioning on $X$.
Strategies for of conditional independence testing may demand much larger sample sizes than marginal independence tests due to the explicitly modelling sub-domains of the values which $Z$ takes.

\subsubsection*{Equality of distribution of two unpaired samples, $X_1 \stackrel{d}{=} X_2$,}
where $\stackrel{d}{=}$ indicates equality in underlying distribution.
An unpaired two-sample test for equality of distribution tests whether two unpaired i.i.d.~samples from $X_1$ and $X_2$ are sampled from the same underlying distributions. The connection to independence testing is not obvious but may be established by pairing each draw from $X_1$ or $X_2$ with a draw from a variable $Y$ taking values in $\{1,2\}$, indicating whether the draw was from $X_1$ or $X_2$, i.e., taking value $i$ if the draw was from $X_i$.

\subsection{The state-of-art in advanced independence testing}

The task of independence testing (in various settings) has been tackled from many different angles.
As soon as one considers observations that are multivariate, or the conditional task, there is not one universally agreed-upon method, but many different approaches. The most prominent ideas used for multivariate and/or conditional independent testing will be presented in this section together with their advantages and shortcomings.

\subsubsection*{Density estimation}
The classical approach. In case of existence, the joint probability density function contains all the necessary information about independence structures for a set of random variables in the multivariate case. While many univariate tests are based on or may be interpreted as being based on some sort of density estimation, for the multivariate and conditional case, density estimation is a difficult problem. One example is shown in~\cite{Dionisio2006}, where density estimation-based information-theoretical measures are used to conduct multivariate marginal independence tests.

\subsubsection*{Copulas}
An approach that is widely used in finance and risk management. Copulas are multivariate probability distributions with uniform marginals. To use copulas for independence tests, one transforms the marginals of a multivariate distribution into uniform distributions, and the resulting copula contains all information about the independence structure. A simple example using the empirical cdf can be found in \cite{genest2004}. Copulas have mainly been used for marginal multivariate independence tests, such as in \citep{Schweizer1981}. The task of testing for conditional independence has been attempted through the use of partial copulas \citep{bergsma2011}, however strong assumptions (linearity) are made on the type of relationship between the variables that are to be tested for independence and the conditioning set. Two-sample testing has largely remain unaddressed, since its application in the financial sector are less relevant. \cite{remillard2009} describes a two-sample test for the estimated copulas, and hence, the independence structures, but the test does not extend to a comparison of the joint distributions as a whole.

\subsubsection*{Kernel methods}
A relatively recent approach using optimization techniques in reproducing kernel Hilbert spaces to answer independence queries based on various test-statistics. Kernel methods were first used for marginal multivariate independence testing via a quantitative measure of dependence, the Hilbert-Schmidt Independence Criterion (HSIC), which is defined as the squared Hilbert Schmidt norm of the cross-covariance operator (the covariance between two difference spaces), see~\cite{Gretton2005}. \cite{Gretton2008} further expands on the theoretical foundations of the HSIC's distribution under the null, resulting in a hypothesis test for the independence of two variables $X$ and $Y$ using the HSIC criterion. \cite{Gretton2010} outlines a non-parametric test for independence, and demonstrates its performance on samples of random variables with up to three dimensions.\\
The conditional case is tackled in~\cite{Zhang2012}, which derives a test statistic based on conditional cross-covariance operators, for which the asymptotic distribution under the null-hypothesis $X  \indep Y | Z$ is derived. They state that the manual choice of kernel can affect type~II and, more importantly, type~I errors, especially when the dimensionality of the conditioning set is large. Additionally, they found that the performance of their method decreases in the dimensions of the conditioning set, constraining the set of problems for which the test is viable for.\\
As for two-sample testing, \cite{Baringhaus2004} derive a test based on the Euclidean inter-point distances between the two samples. A different approach to the same test statistics with the additional use of characteristic functions was made by~\cite{Alba2008}. While the difficulty of density estimation is reduced when using the empirical distributions, the data requirements tend to be much larger, while additionally imposing a sometimes taxing computational complexity. Both of these methods can be related to kernel functions, which was picked up in \cite{Gretton2012}, which proposed kernel-based two-sample test by deriving three multivariate tests for assessing the null-hypothesis that the distributions $p$ and $q$ generating two samples $X\sim p$, $Y \sim q$ are equal, $p = q$. The criterion introduced is the ``Maximum Mean Discrepancy'' (MMD) between $p$ and $q$ over a function space $\mathcal{F}$.
$$\mbox{MMD} = \sup \limits_{f \in \mathcal{F}}( \mathbb{E}_{X \sim p}[f(X)] - \mathbb{E}_{Y \sim q}[f(Y)])$$
That is, the maximum difference between the expected function values with respect to the function space (in their case, a Reproducing Kernel Hilbert Space).

\subsubsection*{Predictive independence testing: precursors}

There are a few instances in which independence testing has already been approached from the perspective of predictability.\\

Most prominently,~\citet{Lopez2017} present a two-sample test (on equality of distribution) based on abstract binary classifiers, e.g., random forests. The presented rationale is in parts heuristic and specific to the two-sample setting. Their main idea is as follows: if the null-hypothesis of the two samples being drawn from the same distribution is true, no classifier assigning to a data point the distribution from it was drawn should fare significantly better than chance, on an unseen test set.\\
The work of~\citet{Lopez2017} may be seen as a special case of our predictive independence test presented later, for $Y$ taking values in a two-element set, and the loss being the misclassification loss. It is probably the first instance in which the abstract inability to predict is explicitly related to independence, albeit somewhat heuristically, and via the detour of two-sample testing for equality in distribution. The predictive independence testing strategy of~\citet{Lopez2017} has recently been generalized to the case of conditional independence testing by~\citet{sen2017model}.\\
In earlier work,~\citet{sriperumbudur2009kernel} already relate the kernel two-sample test to a specific classifier, the Parzen window classifier, evaluated by the misclassification loss. Thus this ideas of~\citet{sriperumbudur2009kernel} may in turn be seen as a precursor of the later work of~\citet{Lopez2017} which abstracts the idea that any classifier could be used to certify for equality of the two samples.\\

In a parallel strain of work in economics literature, by~\citet{guha2016quantile} and~\citet{belloni2017quantile}, a close connection between quantile regression, conditional independence, and graphical modelling has been observed and exploited. Indeed many of the observations made here coincide with the more general correspondence established in our work, for the specific choice of linear models as predictors and quantile losses as convex loss functionals, compare Theorem~\ref{Thm:uninfind-regr}. While the choice of linear predictors prevents a general result characterizing (in)dependence in terms of predictability such as our Theorem~\ref{Thm:uninfind-regr}, the results obtained by~\citet{belloni2017quantile} contain reminiscent statements for the (restrictive) case that linear prediction functionals provide sufficient certificates.

\subsection{Issues in current methodology}

\subsubsection*{Density estimation}
While the probability density function is in some sense optimal for measuring dependence, since it contains all the information about the random variable, its estimation is a difficult task, which requires either strong assumptions or large amounts of data (or both). Furthermore, due to the curse of dimensionality, density estimation for 3 or more dimensions/variables is usually practically intractable, hence its practical usefulness is often limited to testing pairwise independence rather than testing for full (mutual) independence.

\subsubsection*{Copulas}
Leaving aside issues arising by practitioners misunderstanding the method (which had a strong contribution to the 2007/2008 financial crisis, but no bearance whatsoever on the validity of the method when applied correctly), copula-based independence testing is largely a heuristics driven field, requiring many subjective manual choices for estimation and testing. Above all, copula methods require a user to subjectively choose an appropriate copula from a variety of options, such as for example the survival copula, the multivariate Gaussian or the multivariate Student's t-copula \citep{cherubini2004}. Additionally, two-sample testing is to date largely unaddressed in the field.

\subsubsection*{Kernels}
As for the copula-based methods, kernels require many subjective and manual choices, such as the choice of kernel function, and its hyper-parameters. While in a theoretical setting, these choices are made by using the ground truth (and generating artificial data from it), in practice it is difficult to tune the parameters and make statements about the confidence in results. Additionally, the cost of obtaining the test statistic and its asymptotic distribution may be high for many of the proposed approaches. There are attempts at resolving these issues, \citep{Zaremba2013} and \cite{GrettonNips2012} outline heuristics and strategies to minimize the heuristic choices, and \cite{Gretton2009} and \cite{Chwialkowski2015} propose more computationally efficient strategies.\\[0.2in]

While all methods have their merits, and there is vast research on specific problems, providing specific solutions to the individual challenges, what is missing is an approach that is scalable and powerful not just for specific cases, but for the general case, that automatically finds good solutions to real problems, where the ground truth is not known and cannot be used for tuning. 

%% file: 3_pcit-theory.tex
\section{Predictive independence testing}\label{sec:depend}
This section will explore the relationship between predictability and dependence in a set of variables. It will be shown that one can use supervised learning methodology to conduct marginal and conditional independence tests. This distinction between marginal and conditional dependence is not theoretically necessary (since the marginal case can be achieved by setting the conditioning set to the empty set), but is made to highlight specific properties of the two approaches. First, equivalence of independence statements and a specific supervised learning scenario will be shown. After, a routine leveraging this equivalence to test for conditional independence will be proposed.

\subsection{Mathematical setting}
The independence tests will be based on model comparison of supervised learning routines.

Supervised learning is the task where given i.i.d. samples $(X_1,Y_1),\dots, (X_N,Y_N)\sim (X,Y)$ taking values in $\calX\times \calY$, to find a prediction functional $f:\calX\rightarrow \calY$ such that $f(X_i)$ well approximates a target/label variables $Y_i$, where "well" is defined with respect to a loss function $L$ which is to be minimized in expectation.

\begin{Def}\label{Def:convex}
A (point prediction) \emph{loss functional} is an element of $[\calY\times \calY \rightarrow \RR]$, i.e., a function with range $\calY\times \calY$ and image $\RR$. By convention, the first argument will be considered the proposed/predicted value, the second the true value to compare to. By convention, a lower loss will be considered more beneficial.\\
A loss functional $L:\calY\times \calY \rightarrow \RR$ is called:
\begin{enumerate}
\itemsep-0.2em
\item[(i)] \emph{convex} if $L(.,y)$ is lower bounded and $\EE\left[L(Y,y)\right]\le L\left( \EE[Y],y\right)$ for any $\calY$-valued random variable $Y$ and any $y\in \calY$, and
\item[(ii)] \emph{strictly convex} if $L(.,y)$ is lower bounded and $\EE\left[L(Y,y)\right]\lneq L\left( \EE[Y],y\right)$ for any non-constant $\calY$-valued random variable $Y$ and any $y\in \calY$.
\end{enumerate}
\end{Def}

More formally, a good prediction functional possesses a small expected generalization error
$$\epsilon_L(f) := \mathbb{E}[L(f(X),Y)].$$
In usual practice and in our setting $f$ is estimated from training data
$$\calD = \{(X_i, Y_i)\}_{i = 1}^N,$$ hence it is in fact a random object.
The generalization error may be estimated from test data
$$\calT = \{(X_i^*, Y_i^*)\}_{i = 1}^M\quad\mbox{where}\;(X^*_1,Y^*_1),\dots,(X^*_M,Y^*_M)\sim (X,Y),$$
which we assume i.i.d.~and independent of the training data $\calD$ and $f$, as
$$\widehat{\epsilon}_L(f) = \sum_{i=1}^M[L(f(X_i^*),Y_i^*)].$$
Independence of training data $\calD$ and test data $\calT$ is required to avoid bias of the generalization error estimate (``over-fitting''). In the following sections, when estimating the generalization error, it will be assumed that $f$ is independent of the test data set, and we condition on $\calD$, hence treat $f$ as fixed (i.e., conditional on $\calD$).

In the following, we assume that $\calY$ is a vector space, and $\calY \subseteq \RR^q$, which naturally includes the setting of supervised regression. However, it also includes the setting of supervised classification through a special choice of loss:

\begin{Rem}\label{Rem:class}
Deterministic classification comprises the binary case of $\calY = \{-1,1\}$, or more generally $\calY$ being finite.
In neither case are additions or expectations defined on $\calY$, as $\calY$ is just a discrete set.
Supervised classification algorithms in the deterministic case are asked to produce a prediction functional $f:\calX\rightarrow \calY$, and
are usually evaluated by the misclassification loss or 0/1-loss
$$L:\calY\times \calY\rightarrow \RR\;:\;(y,y_*)\mapsto \mathbbm{1}[y \neq y_*],$$
where $\mathbbm{1}[y \neq y_*]$ is the indicator function which evaluates to $0$ if $y=y_*$ and to $1$ otherwise.
Hence, Definition~\ref{Def:convex} of convexity does not directly apply to the misclassification loss as expectations are not defined on $\calY$.\\
However, by allowing all algorithms to make predictions which are $\calY$-valued random variables (instead of constants), one finds that
$$\EE[L(Y,y)] = L'(p,y)\quad\mbox{where}\;L':(p_Y,y)\mapsto 1-p_Y(y)$$
and $p_Y$ is the probability mass function of $Y$.
Identifying $\calY$-valued random variables with their probability mass functions, one may replace
\begin{enumerate}
\itemsep-0.2em
\item[(i)] $\calY$ by the corresponding subset $\calY'$ of $\RR^{\# \calY -1}$ which is the set of probability vectors (the so-called probability simplex).
For example, $\calY = \{-1,1\}$ would be replaced by $[0,1]$.
\item[(ii)] the misclassification loss by the probabilistic, convex (but not strictly convex) loss $L':\calY'\times \calY'\rightarrow \RR$, where the observations in the second $\calY'$ argument are always pmf describing constant random variables.
\end{enumerate}
A further elementary calculation (see appendix~\ref{app:classifdet} for an explicit derivation) shows that $L'$ is always minimized by making deterministic predictions:
$$\underset{p\in\calY'}{\argmin}\; \EE[L'(p,Y)] \cap [\calY\rightarrow \{0,1\}]\;\mbox{is non-empty,}$$
i.e., the $L'$-best classifier may always be chosen to be a deterministic one, i.e., one that always predicts a probability mass functions with probabilities $0$ or $1$.\\
This exhibits deterministic classification
as a special case of probabilistic classification with a special choice of loss function.
\end{Rem}

\subsection{Elicitation by convex losses}

Loss functionals are canonically paired with summary statistics of distributions:

\begin{Def}\label{Def:elicit}
Let $L:\calY\times \calY\rightarrow \RR$ be a convex loss functional.
For a $\calY$-valued random variable $Y$, we define
$$\mu_L([Y]):= \underset{y\in\calY}{\argmin} \EE\left[L(y,Y)\right]$$
where $[Y]$ denotes $Y$ as a full object (i.e., a measurable function), rather than its value.\\
Following~\citet{gneiting2007strictly}, we term the functional which maps $\calY$-valued distributions to sub-sets of $\calY$ the \emph{eliciting functional} associated to $L$.
We call $\mu_L([Y])$ the \emph{summary} of $Y$ \emph{elicited by} $L$.
\end{Def}

Note that well-definedness, i.e., existence of the minimizer, is ensured by convexity of $L$ (and the implied continuity).
If $L$ is strictly convex, there is furthermore a unique minimizer, in which case we will exchangeably consider $\mu_L$ to be a functional with target $\calY$.

Well-known examples of elicited summaries are given in the following:

\begin{Lem}\label{Lem:elicited}
The following relations between losses and elicited statistics of real-valued random variables hold:
\begin{enumerate}
\itemsep-0.2em
\item[(i)] the (strictly convex) squared loss $L:(y,y_*)\mapsto (y-y_*)^2$ elicits the mean. That is, $\mu_L([Y]) = \EE [Y]$ for any $\RR^n$-valued random variable $Y$.
\item[(ii)] the (convex but not strictly convex) absolute loss $L:(y,y_*)\mapsto |y-y_*|$ elicits the median(s). That is, $\mu_L([Y]) = \mbox{median} [Y]$ for any $\RR$-valued random variable $Y$.
\item[(iii)] the (convex but not strictly convex) quantile-loss (or short: Q-loss) $L(y,y_*)=\alpha\cdot m(y_*,y) + (1-\alpha)\cdot m(y,y_*)$, with $m(x,z)=\min(x-z,0)$, elicits the $\alpha$-quantile(s).
            That is, $\mu_L([Y]) = F^{-1}_Y(\alpha)$ for any $\RR$-valued random variable $Y$, where $F^{-1}_Y:[0,1] \rightarrow P(\RR)$ is the set-valued inverse c.d.f.~of $Y$ (with the convention that the full set of inverse values is returned at jump discontinuities rather than just an extremum).
\end{enumerate}
\end{Lem}
\begin{proof}
(i) after substitution of definition, the claim is equivalent to the statement to the mean being the minimizer of squared distances. A more explicit proof is given in Appendix~\ref{app:sqloss}.\\
(ii) follows, by setting $\alpha = \frac{1}{2}$, from (iii).\\
(iii) This is carried out in Appendix~\ref{app:Qloss}.
\end{proof}

In the supervised setting, the best possible prediction functional can be exactly characterized in terms of elicitation:

\begin{Prop}\label{Prop:bestpred}
Let $L$ be a (strictly) convex loss. Then, it holds that
$$\underset{f\in [\calX\rightarrow \calY]}{\argmin}\; \varepsilon_L(f) = \left[x\mapsto \mu_L[Y|X=x]\right].$$
That is, the best possible prediction as measured by $L$ is predicting the statistic which $L$ elicits from the conditional random variable $Y|X=x$.
\end{Prop}
\begin{proof}
The prediction functional $\varpi: x\mapsto \mu_L[Y|X=x]$ is well-defined, hence it suffices to show that $\varepsilon_L(f)\ge \varepsilon_L(\varpi)$
for any prediction functional $f:\calX\rightarrow \calY$.\\
Now by definition of $\mu_L$, it holds that
$$\EE \left[L(\varpi(X),Y) |X\right]\le \EE \left[L(f(X),Y) |X\right].$$
Taking total expectations yields the claim.
\end{proof}

Intuitively, the best prediction functional, as measured by a convex loss $L$, always predicts the statistic elicited by $L$ from the conditional law $[Y|X=x]$.

\subsection{Predictive uninformedness}\label{sec:depend.uninformed}

We will now introduce the notion of an uninformed baseline which will act as a point of comparison.

\begin{Def}
A prediction functional $f:\calX \rightarrow \calY$ is called uninformed if it is a constant functional, i.e., if $f(x) = f(y)$ for all $x,y\in\calX$. We write $u_\alpha$ for the uninformed prediction functional $u_\alpha:x\mapsto \alpha$.
\end{Def}

We will show that, for a given loss function, there is one single uninformed baseline that is optimal.

\begin{Lem}
\label{Lem:bestun}
Let $L$ be a (strictly) convex loss, let $\mu:= \mu_L([Y])$ be the/a statistic elicited by $L$ (see Definition~\ref{Def:elicit}).
Then, the following quantities are equal:
\begin{enumerate}
\itemsep-0.2em
\item[(i)] $\inf \{\varepsilon_L(f)\;:\; f\mbox{ is an uninformed prediction functional}\}$
\item[(ii)] $\varepsilon_L(u_\mu)$
\end{enumerate}
That is, $u_\mu$ is achieves the lowest possible ($L$-)loss amongst uninformed prediction functionals \emph{and} prediction strategies.
\end{Lem}
\begin{proof}
Note that $\varepsilon (u_\alpha) = \EE[L(\alpha,Y)]$.
It follows hence by definition of $\mu$ that $\varepsilon(u_\mu)\le \varepsilon(u_\alpha)$ for any (constant) $\alpha\in \calY$. I.e., $u_\mu$ is the best uninformed prediction \emph{functional}.
\end{proof}

Lemma~\ref{Lem:bestun} motivates the definition of the best uninformed predictor:

\begin{Def}\label{Def:bestun}
We call $u_\mu$, as defined in Lemma~\ref{Lem:bestun}, the ($L$-)best uninformed predictor (even though it may not be unique, the choice, when possible, will not matter in what follows).\\
We call a prediction functional ($L-$)better-than-uninformed if its expected generalization loss is strictly smaller than of the ($L-$)best uninformed predictor. More formally, a prediction functional $f$ is $L$-better-than-uninformed if $\varepsilon_L(f)\lneq \varepsilon_L(u_\mu)$.
\end{Def}

For convenience, we further introduce some mathematical notation for best predictors:

\begin{Not}
Let $L$ be a (strictly) convex loss. We will write:
\begin{enumerate}
\itemsep-0.2em
\item[(i)] $\varpi^{(L)}_Y:=\left[x\mapsto \mu_L[Y]\right]$ for the/a ($L$-)best uninformed predictor as defined in Definition~\ref{Def:bestun}.
\item[(ii)] $\varpi^{(L)}_{Y|X}:=\left[x\mapsto \mu_L[Y|X=x]\right]$ for the/a ($L$-)best predictor as considered in Proposition~\ref{Prop:bestpred}.
\end{enumerate}
$\varpi^{(L)}_Y$ and $\varpi^{(L)}_{Y|X}$ are unique when $L$ is strictly convex, as per the discussion after Definition~\ref{Def:elicit}.
When multiple choices are possible for the minimizer, i.e., if $L$ is convex but not strictly convex, an arbitrary choice may be made (not affecting subsequent discussion).
The superscript $L$ may be omitted in situations where the loss is clear from the context.
\end{Not}

An important fact which we will use in testing is that if a better-than-uninformed prediction functional exists, then $\varpi_{Y|X}$ is an example:

\begin{Prop}
\label{Prop:uninfind}
Fix a (strictly) convex loss. The following are equivalent:
\begin{enumerate}
\itemsep-0.2em
\item[(i)] $\varpi_{Y|X}$ is $L$-better-than-uninformed.
\item[(ii)] $\varpi_{Y|X}$ is not $L$-uninformed.
\item[(iii)] There is an $L$-better-than-uninformed prediction functional.
\end{enumerate}
Note that equivalence of (i) and (ii) is not trivial, since there are prediction functionals which are not better-than-uninformed but not uninformed.
\end{Prop}
\begin{proof}
The equivalence of (i) and (iii) follows directly from Lemma~\ref{Lem:bestun}, and noting that (i) implies (iii).\\
(i)$\Leftrightarrow$(ii): By Proposition~\ref{Prop:bestpred}, $\varepsilon \left(\varpi_{Y|X}\right)\le \varepsilon (f)$ for any $f$, in particular also any uninformed $f$. By Lemma~\ref{Lem:bestun} and the above, the inequality is strict if and only if $\varpi_{Y|X}$ is better-than-uninformed.
\end{proof}

\subsection{Statistical dependence equals predictability}\label{sec:depend.unconditional}

We continue with a result that relates - or more precisely, equates - better-than-uninformed predictability with statistical dependence.
As such, it shows that the choice of constant functions as a proxy for uninformedness was canonical.

We start with the more intuitive direction of the claimed equivalence:

\begin{Prop}\label{Prop:predind}
Let $X,Y$ be random variables taking values in $\calX,\calY$. Let $L$ be any convex loss functional.
If $X,Y$ are statistically independent, then:
There exists no $L$-better-than-uninformed prediction functional for predicting $Y$ from $X$.
More precisely, there is no prediction functional $f:\calX\rightarrow \calY$ and no convex loss $L:\calY\times \calY\rightarrow \RR$ such that $f$ is $L$-better-than-uninformed.
\end{Prop}
\begin{proof}
Assume $X,Y$ are statistically independent. Let $L$ be a convex loss function, let $f:\calX\rightarrow \calY$ be a prediction functional.
Then, by convexity of $L$,
$$\EE\left[L(f(X),Y)|Y\right]\ge L(\EE [f(X)],Y)|Y.$$
Since $X$ is independent of $Y$ (by assumption) and $f$ is not random, it holds that
$\EE [f(X)|Y] = \EE [f(X)]$, i.e., $\EE [f(X)|Y=y]$, as a function of $y$, is constant.
Writing $\nu:= \EE [f(X)]$, we hence obtain that
$$L(\EE [f(X)],Y)|Y = L(\nu ,Y) = L(u_\nu,Y).$$
After taking total expectations, it hence holds by the law of total expectation that
$$\EE\left[L(f(X),Y)\right]\ge \EE\left[ L(u_\nu,Y)\right],$$
meaning that $f$ is not better-than-uninformed w.r.t.~$L$. Since $f$ and $L$ were arbitrary, the statement holds.
\end{proof}

Proposition~\ref{Prop:predind} states that independence implies unpredictability (as measured per expected loss).
A natural question to ask is whether the converse holds, or more generally which converses exactly, since the loss functional $L$ in (ii) and (iii) of Proposition \ref{Prop:uninfind}, was arbitrary, and independence in (i) is a statement which remains unchanged when exchanging $X$ and $Y$, while predictability is not symmetric in this respect. Hence the weakest possible converse would require unpredictability w.r.t.~all convex losses and w.r.t.~either direction of prediction, however much stronger converses may be shown.

Before stating the mathematically more abstract general result for the converse direction, we first separately present special cases for the three important sub-cases, namely deterministic classification, probabilistic classification, and regression.

\begin{Thm}
\label{Thm:uninfind-classdet}
As in our setting, consider two random variables $X,Y$, taking values in $\calX$ and $\calY$. Assume $\calY = \{-1,1\}$.
The following are equivalent:
\begin{enumerate}
\itemsep-0.2em
\item[(i)] $X$ and $Y$ are statistically independent.
\item[(ii)] There exists no better-than-uninformed prediction functional predicting $Y$ from $X$.
More precisely, there is no prediction functional $f:\calX\rightarrow \calY$ such that $f$ is $L$-better-than-uninformed for
the misclassification loss $L:(p,y)\mapsto 1-p(y)$.
\end{enumerate}
Regarding the specific loss $L$, see Remark~\ref{Rem:class} for the identification of a class prediction with a probability-1 probabilistic prediction.
\end{Thm}
\begin{proof}
By Proposition~\ref{Prop:predind}, the only direction which remains to be proven is (ii)$\Rightarrow$(i):
(i)$\Rightarrow$(ii) follows directly from substituting the specific $L$ into the implication between statements with the same numbers in Proposition~\ref{Prop:predind}.\\

(ii)$\Rightarrow$ (i): We prove this by contraposition: we assume $X$ and $Y$ are statistically dependent and construct a better-than-uninformed prediction functional.\\
By the equivalence established in Remark~\ref{Rem:class}, the best uninformed predictor $\varpi_Y$ is predicting the most probable class, w.l.o.g. $1\in\calY$, and its expected generalization loss is one minus its generative frequency $\varepsilon(\varpi_Y) = P(Y=-1)$.\\
Since $X$ and $Y$ are statistically dependent, there is a positive probability $\calX'\subseteq \calX$ (measurable with positive probability measure) such that $P(Y=1|X\in\calX')\ge P(Y=-1|X\in\calX')$ (the definition yields $\neq$, but w.l.o.g.~replacing $\calX'$ by its complement yields $\ge$).
An elementary computation shows that the prediction functional $f:\calX\rightarrow \calY\;:\; 2\cdot \mathbbm{1}[x\in\calX']-1$ is better-than-uninformed.
\end{proof}

The proof of Theorem~\ref{Thm:uninfind-classdet} shows that for a converse, one does not necessarily need to consider predictability of $X$ from $Y$ as well.
However, the deterministic binary classification case is somewhat special, since the misclassification loss is insufficient in the multi-class case,
and in general a single loss function will be unable to certify for independence.
In order to formulate these negative results, we define shorthand terminology for stating correctness of the converse.

\begin{Def}
Fix a (label) domain $\calY$, let $\calL$ be a set with elements being convex loss functionals in $[\calY\times\calY\rightarrow \RR]$.
We call $\calL$ \emph{faithful} for $\calY$ if for every statistically dependent pair of random variables $X,Y$ taking values in $\calX,\calY$,
there is $L\in\calL$ and a prediction functional $f:\calX\rightarrow \calY$ such that $f$ is $L$-better-than-uninformed. If $\calL$ is a one-element-set, we call its single element faithful for $\calY$.\\
If $\calL$ is faithful for $\calY$, and $\calL$ is endowed with a measure $\mu$ such that no $\mu$-strict sub-set of $\calL$ is faithful for $\calY$, then we call $\calL$ ($\mu$-)strictly faithful.\\
If $\calY$ is canonically associated with a prediction task (such as classification for finite $\calY$ or class-probabilities $\calY$, regression for continuous $\calY$), the reference to $\calY$ may be replaced by a reference to that task.
\end{Def}

In this terminology, Theorem~\ref{Thm:uninfind-classdet} states that the misclassification loss is faithful for the label domain $\calY=\{-1,1\}$, or equivalently that the misclassification loss is faithful for deterministic binary classification (and strictly faithful, since any set smaller than an one-element set, as by the counting measure, is empty and by definition not faithful).
We can now state some negative and positive results on deterministic multi-class classification and probabilistic classification:

\begin{Prop}
\begin{enumerate}
\itemsep-0.2em
\item[(i)] The misclassification loss $L:(p,y)\mapsto 1-p(y)$ is \emph{not} faithful for deterministic multi-class classification, i.e., for $\calY$ being finite and containing 3 or more elements.
\item[(ii)] The log-loss $L:(p,y)\mapsto -\log(p(y))$ is (strictly) faithful for probabilistic classification.
\item[(iii)] The squared probabilistic classification loss $L:(p,y)\mapsto (1-p(y))^2$ is (strictly) faithful for probabilistic classification.
\item[(iv)] The Brier loss $L:(p,y)\mapsto (1-p(y))^2 + \sum_{y'\neq y}p(y')^2$ is (strictly) faithful for probabilistic classification.
\end{enumerate}
\end{Prop}
\begin{proof}
(i): It suffices to construct a counterexample for each possible $\calY$, i.e., every finite $\calY$ with 3 or more elements.
Let $\calX = \{-1,1\}$, and let $y\in\calY$ be arbitrary. Define $X,Y$ such that $P(Y=y|X=1) = P(Y=y|X=-1) := 0.9$ and
$P(Y=y'|X=1) \neq P(Y=y'|X=-1)$ for some class $y'\neq y$. This choice is possible as $\calY$ has 3 or more elements.
The best uninformed predictor always predicts $y$, with expected generalization loss $0.1$, while it cannot be outperformed, see e.g.~the discussion in Remark~\ref{Rem:class}.\\
(ii)-(iv): for faithfulness, it suffices to show that $\varepsilon(\varpi_{Y|X}) = \varepsilon(\varpi_{Y})$ implies statistical independence of $X,Y$.\\
Explicit computations in each case (see Appendices~\ref{app:logloss} and~\ref{app:brier}) show that $\argmin_p\EE[L(p,Y)]=p_Y$, where $p_Y$ is the probability mass function of $Y$, implying that $\varpi_Y= [x\mapsto [y\mapsto P(Y=y)]]$.
By conditioning the same statement on $X$, this also implies that $\argmin_p \EE[L(p,Y)|X] = [y\mapsto P(Y=y|X)],$ thus $\argmin_f \EE[L(f(X),Y)] = [y\mapsto P(Y=y|X=x)]$.
In particular, $\varpi_{Y|X} = [x\mapsto [y\mapsto P(Y=y|X=x)]]$ is the unique minimizer of the expected generalization loss.
Thus, $\varepsilon(\varpi_{Y|X}) = \varepsilon(\varpi_{Y})$ only if both functions are identical, i.e., $P(Y=y|X=x) = P(Y=y)$ for all $x$, which is one possible definition of $X$ and $Y$ being statistically independent.
\end{proof}

For regression, the usual loss functions are unable to certify for independence anymore:

\begin{Prop}\label{Prop:uninfnegreg}
The following convex loss functions (taken as single-element sets) are not faithful for univariate regression, i.e., for the label domain $\calY = \RR$
\begin{enumerate}
\itemsep-0.2em
\item[(i)] the squared loss $L(y,y_*) = (y-y_*)^2$
\item[(ii)] the absolute loss $L(y,y_*) = |y-y_*|$
\item[(iii)] the distance loss $L(y,y_*)=d(y,y_*)^2$, for any metric $d: \calY\times \calY \rightarrow \RR$
\end{enumerate}
\end{Prop}
\begin{proof}
It suffices to construct, for each convex loss functional $L$, a counterexample where $X,Y$ are statistically dependent, but no prediction functional predicting $Y$ from $X$ is better-than-$L$-uninformed.\\
In each case, one may construct two $\calY$-valued random variables $Y_1,Y_2$ with distinct laws such that $\mu_L[Y_1] = \mu_L[Y_2]$ - for example, an arbitrary non-constant $Y_1$ and the $Y_2$ being constant $\mu_L[Y_1]$. Further setting $\calX = \{-1,1\}$ and defining a non-constant $\calX$-valued random variable $X$, together with an $\calY$-valued random variable such that $Y_i=Y|X=i$ for $i\in\calX$ yields an example of a statistically dependent pair $X,Y$ of random variables where the constant prediction of $\mu_L[Y_1] = \mu_L[Y_2]$ is not only the $L$-best uninformed prediction functional, but also the $L$-best prediction functional.\\
Using equivalence of (i) and (iii) in Proposition~\ref{Prop:uninfind} proves the claim.
\end{proof}

The previously introduced quantile losses form a strictly faithful set of losses for univariate regression:

\begin{Thm}
\label{Thm:uninfind-regr}
The set of Q-losses is strictly faithful for (univariate) regression.\\
More precisely, the set $\calL = \{L_\alpha\;:\; \alpha \in [0,1]\}$, where $L_\alpha(y,y_*)=\alpha\cdot m (y_*,y) + (1-\alpha)\cdot m(y,y_*)$, with $m(x,z)=\min(x-z,0)$
(see Lemma~\ref{Lem:elicited}~(iii)), endowed with the Lebesgue measure through canonical identification of $L_\alpha$ with $\alpha\in[0,1]$, is strictly (Lebesgue-measure-)faithful for $\calY = \RR$.
\end{Thm}
\begin{proof}
For faithfulness, we show that impossibility of $(L_\alpha)$-better-than-uninformed prediction of a $\RR$-valued random variable $Y$ from an $\calX$-valued random variable $X$ implies statistical independence of $X$ and $Y$.\\
Thus, assume that there is no $\alpha\in \RR$ and no prediction functional $f$ such that $f$ is $L_\alpha$-better-than-uninformed. By equivalence of (ii) and (iii) in Proposition~\ref{Prop:uninfind} and negation/contraposition, $\varpi_{Y|X}^{(L_\alpha)}$ is uninformed hence constant. By Lemma~\ref{Lem:elicited}~(iii), $\mu_{L_\alpha}[Y|X=x]$ is the $\alpha$-quantile of the conditional random variable $Y|X=x$, which by the previous statement does not depend on $x$. Since $\alpha$ was arbitrary, none of the quantiles of $Y|X=x$ depends on $x$, i.e., the cdf and hence the laws of all conditionals $Y|X=x$ agree, which implies (by one common definition/characterization of independence) that $X$ and $Y$ are statistically independent.\\
Strict faithfulness follows by following through the above argument after removing a positive-measure open set $U\subseteq [0,1]$ from the indices, i.e., $L_\alpha$ for $\alpha_\in U$ from $\calL$. As $U$ has positive measure, we may pick $u\in U$ and $\calX = \{-1,1\}$ as well as conditional cdf such that $P(Y\le u|X = 1)\neq P(Y\le u|X = -1)$ while $P(Y\le x|X=1) = P(Y\le X|X=-1)$ for all $x\not\in U$. By the above argument, predicting the $\alpha$-quantile of $Y$ is the $L_\alpha$-best prediction functional from $X$ for any $\alpha\not\in U$, and it is furthermore a uninformed prediction strategy, thus $\calL$ is not faithful after removing $L_\alpha,\alpha\in U$.
\end{proof}

In the light of Theorem~\ref{Thm:uninfind-regr}, it may be interesting to ask for a theoretical characterization of a strictly faithful set of losses for univariate regression (e.g., does it need to be infinite?), or what may be a set of strictly faithful losses for multivariate regression.\\
A-priori, it is even unclear whether there is a set of (not necessarily strictly) faithful losses for general prediction tasks, which the following result answers positively:

\begin{Thm}
\label{Thm:uninfind}
Assume that $\calY$ may be identified (including the taking of expectations) with a sub-set of $\RR^n$, for some $n$, e.g., in multivariate regression, or simultaneous prediction of multiple categorical and continuous outputs. Then, the set of convex losses is faithful for $\calY$.
\end{Thm}
\begin{proof}
Consider random variables $X$ and $Y$ taking values in $\calX$ and $\calY$, are statistically dependent.
By definition of (in)dependence, this is equivalent to there existing $\calX'\subseteq \calX,\calY'\subseteq \calY$ (measurable with positive joint probability measure) such that
$P(Y\in\calY'|X\in\calX')\neq P(Y\in\calY')$. By taking (or not taking) the complement of $\calX'$ within $\calX$, we may assume without loss of generality that
$P(Y\in\calY'|X\in\calX')\gneq P(Y\in\calY')$.\\
Define $g:\calY\rightarrow \RR\;;\; x\mapsto x^2$ (where we have used the identification of $\calY$ with $\RR^n$.
Define
$$L:\calY\times \calY\rightarrow \RR\;;\; (y,y_*)\mapsto g(y)\cdot \mathbbm{1}(y_*\in\calY') + g(y-\alpha)\cdot \mathbbm{1}(y_*\not\in\calY'),$$
where $\mathbbm{1}(y_*\in\calY')$ is the indicator function for $y_*\in\calY'$, and $\alpha\in\calY\setminus\{0\}$ is arbitrary.
Define $f: \calX\rightarrow \calY\;;\; x\mapsto 0 \mbox{ if }(x\in\calX'),\mbox{otherwise }\alpha$.
An elementary computation shows that $\varepsilon_L(f)$ is better-than-uninformed, hence we have proved non-(ii).
\end{proof}

For the general case, it seems interesting to ask what would be a strictly faithful set of losses, how such sets may be characterized, or whether they even exist (which seems neither obviously nor directly implied by the existence of a faithful set of losses).

Due to the constructive nature of respective proofs (or, more precisely, the semi-constructivity of the proof for multi-variate regression in Theorem~\ref{Thm:uninfind}), model comparison procedures suggested by Theorems~\ref{Thm:uninfind-classdet} and~\ref{Thm:uninfind-regr} on univariate classification and regression will be used in the testing procedures. For convenience we briefly repeat the main results used in testing as a corollary:

\begin{Cor}
\label{Cor:testing}\label{th: indepconstant}
Consider two random variables $X,Y$, taking values in $\calX$ and $\calY$, where $\calY$ is finitely supported pmf (classification) or where $\calY\subseteq \RR^q$ (regression).
The following are equivalent:
\begin{enumerate}
\itemsep-0.2em
\item[(i)] $X$ and $Y$ are statistically dependent.
\item[(ii)] $\varepsilon (\varpi^{(L)}_{Y|X})\lneq \varepsilon_L (\varpi^{(L)}_{Y})$
for $L$ the log-loss/Brier-loss (classification), resp.~ for some convex loss $L$ (regression).
I.e., the $L$-best predictor is $L$-better-than-uninformed, for some $L$.
\item[(iii)] there exists a prediction functional $f:\calX\rightarrow\calY$ such that
$\varepsilon_L (f)\lneq \varepsilon_L (\varpi^{(L)}_{Y})$
for $L$ the log-loss/Brier-loss (classification), resp.~ for some convex loss $L$ (regression).
I.e., there exists an $L$-better-than-uninformed prediction functional, for some $L$.
\end{enumerate}
\end{Cor}

Since Corollary~\ref{Cor:testing}~(i) is the alternative hypothesis in an independence test, we may use a test for the equivalent hypothesis Corollary~\ref{Cor:testing}~(iii), i.e., comparing the performance a learnt $f$ with the best uninformed baseline, as an independence test. Any choice of $f$ is sufficient for the methodology outlined in the latter Section~\ref{section: pseudocode}. Since the null-hypothesis is that $X \indep Y$, choosing a bad $f$ or failing to detect multivariate dependence (a specific case) only decreases the power of the test, while the type 1-error remains unaffected.

\subsection{Conditional independence}
The last section established a theoretical foundation for marginal independence testing, now this foundation will be expanded by adding a framework for testing conditional independence.\\
The statement one would like to make thus connects two random variables, $X$ and $Y$, that are conditionally independent given a third variable, $Z$, taking values in $\calZ$, with the expected generalization loss from predicting $Y$ from $Z$ and from the set $\{X,Z\}$.

In slight extension of our setting, we consider prediction functionals in $[\calX\times \calZ\rightarrow \calY],$ i.e., we separate the features in two parts corresponding to $X$ and the conditioning $Z$.

We generalize the main definitions from Section~\ref{sec:depend.uninformed} to the conditional case:

\begin{Def}
A prediction functional $f:\calX \times \calZ \rightarrow \calY$ is called conditionally uninformed if it does not depend on the first argument, i.e., if $f(x,z) = f(y,z)$ for all $x,y\in\calX$ and $z\in\calZ$.
By notational convention, a functional $g:\calZ\rightarrow \calY$ is identified with the conditionally uninformed prediction functional $(x,z)\mapsto g(z)$.
\end{Def}

In the text, it will be implicitly understood that conditioning will happen on the second argument.
We define straightforward conditional generalizations of baselines and best predictors:

\begin{Def}
We define the following prediction functionals in $[\calX \times \calZ \rightarrow \calY:$\\
The best conditionally uninformed prediction $\varpi_{Y|Z}^{(L)}: (x,z)\mapsto \mu_L([Y|Z=z])$\\
The best conditional prediction $\varpi_{Y|X,Z}^{(L)}: (x,z)\mapsto \mu_L([Y|X=x,Z=z])$
\end{Def}

It is briefly checked that $\varpi_{Y|Z}^{(L)}$ and $\varpi_{Y|X,Z}^{(L)}$ have the properties their names imply:

\begin{Lem}
\label{Lem:bestuncond}
Let $L$ be a (strictly) convex loss. The following holds:
\begin{enumerate}
\itemsep-0.2em
\item[(i)] $\varepsilon_L(\varpi_{Y|Z}^{(L)}) = \min \{\varepsilon_L(f)\;:\; f\mbox{ is a conditionally uninformed prediction functional}\}$
\item[(ii)] $\varepsilon_L(\varpi_{Y|X,Z}^{(L)}) = \min \{\varepsilon_L(f)\;:\; f\in [\calX\times\calZ\rightarrow\calY]\}$
\end{enumerate}
\end{Lem}
\begin{proof}
(i) Conditioning on the event $Z=z$, Lemma~\ref{Lem:bestun} for the unconditional case implies the equality conditional for the event. Since $z$ was arbitrary, it holds without the conditioning.\\
(ii) This is directly implied by Proposition~\ref{Prop:bestpred}, by substituting the joint random variable $(X,Z)$ for the $X$ in Proposition~\ref{Prop:bestpred}.
\end{proof}

With these definitions, the conditional variant of Theorem~\ref{Thm:uninfind} can be stated:

\begin{Thm}
\label{Thm:uninfind-conditional}\label{theorem2}
As in our setting, consider three random variables $X,Y,Z$, taking values in $\calX, \calY, \calZ$, where $\calY$ is finitely supported pmf (classification) or where $\calY\subseteq \RR^q$ (regression).
The following are equivalent:
\begin{enumerate}
\itemsep-0.2em
\item[(i)] $X$ and $Y$ are statistically dependent conditional on $Z$.
\item[(ii)] $\varepsilon_L (\varpi^{(L)}_{Y|(X,Z)})\lneq \varepsilon_L (\varpi^{(L)}_{Y|Z})$
for $L$ the log-loss/Brier-loss (classification), resp.~ for some convex loss $L$ (regression)
%\item[(ii)] There exists no convex loss $L$, and no prediction functional $f:\calX\rightarrow \calY$ predicting $Y$ from $X$, such that $f$ is $L$-better than the best prediction functional $\varpi^{(L)}_{Y|Z}$. I.e., there is no $L$-better-than-$Z$-conditionally-uninformed prediction functional.
\item[(iii)] there exists a prediction functional $f:\calX\times \calZ\rightarrow\calY$ such that $\varepsilon_L (f)\lneq \varepsilon_L (\varpi^{(L)}_{Y|Z})$
for $L$ the log-loss/Brier-loss (classification), resp.~ for some convex loss $L$ (regression)
\item[(iv)] there exists a prediction functional $f:\calX\times \calZ\rightarrow\calY$ such that for all conditionally uninformed prediction functionals $g:(\calX\times)\calZ\rightarrow \calY$, one has $\varepsilon_L (f)\lneq \varepsilon_L (g)$
for $L$ the log-loss/Brier-loss (classification), resp.~ for some convex loss $L$ (regression)
\end{enumerate}

The set of losses in which existence is required in (ii) may be restricted to a set of losses which is faithful for the unconditional setting (such as: quantile losses for $\calY = \RR$ as per Theorem~\ref{Thm:uninfind-regr}), as in Section~\ref{sec:depend.unconditional}, without changing the fact that the stated equivalence is correct.
\end{Thm}
\begin{proof}
The proof is an analogue of that of Corollary~\ref{Cor:testing}.
It may be obtained from following the whole proof through while in addition conditioning on $Z=z$, and noting that (i) holds if and only if the conditionals $X|Z=z$ and $Y|Z=z$ are statistically dependent for some $z$.
\end{proof}

Our conditional independence test is based on statement (iv) in Theorem~\ref{Thm:uninfind-conditional}. Unlike in the unconditional case, there is in general no easy way to estimate $\varpi^{(L)}_{Y|Z}$ directly - as it is equivalent to the generic supervised learning task. This is in line with the usual approach to and caveats of supervised learning - if there were a universal way to estimate $\varpi^{(L)}_{Y|Z}$, that would amount to a universally perfect supervised learning strategy. Thus, in the algorithm in Section \ref{section: pseudocode}, the statement will be replaced by a slightly weaker statement, where automatic model selection will determine an $f$ as well as a $g$, in general without guarantees that $f$ estimating $\varpi^{(L)}_{Y|X,Z}$ and $g$ estimating $\varpi^{(L)}_{Y|Z}$ are optimal, but possibly with guarantees for specific classes of $f$ and $g$.

\subsection{Testing independence through baseline model comparison}
\label{section: signiftest}

Theorems~\ref{Thm:uninfind} or~\ref{Thm:uninfind-conditional} in the previous sections establish a new basis for (marginal or conditional) independence testing.
Namely, the theorems relate testing of (marginal or conditional) dependence between $X$ and $Y$ to testing predictability of $Y$ from $X$.
Namely, Theorems~\ref{Thm:uninfind} or~\ref{Thm:uninfind-conditional} state the following:
If there exists a significantly better-than-(conditionally-)uninformed prediction strategy $f$, i.e,
\begin{equation}\label{eq: betterpred} \varepsilon_L(f) \lneq \varepsilon_L(g), \end{equation}
where $g$ is a suitable baseline (an approximation of $\varpi^{(L)}_{Y}$ or $\varpi^{(L)}_{Y|Z}$),
then we may conclude by Theorem~\ref{Thm:uninfind} or~\ref{Thm:uninfind-conditional} that $X$ and $Y$ are not (marginally or conditionally) independent.

Thus the problem of independence testing is reduced to the problem of testing whether there exists a prediction functional $f$ which outperforms the baseline $g$ as measured by some convex loss function $L$.

We make a few remarks about the strategy and logic of our proposed significance test.
\begin{itemize}
\item Neither the proposed functional $f$ nor the baseline $g$ are a-priori known in the usual practical setting, and neither is an $L$ which may make a difference apparent. However, both $f$ and $g$ may be seen to approximate a best prediction functional which is unknown, thus as instances of supervised learning. Hence, some choice has to be made, in absence of the ``perfect learning strategy''. A bad choice will contribute to a loss of power only if type I error control is achieved.
\item The baseline $g$ approximates $\varpi^{(L)}_{Y}$ in the marginal case. For frequently used $L$, the baseline $g$ predicts constants which are the mean or the median or other well-studied statistics of a sample from $Y$, hence $g$ may have a beneficial asymptotic.
\item Significance tests to compare prediction strategies are studied in~\cite{nadeau2003}, which proposes amongst others a paired sample test of prediction residuals.
\end{itemize}

The values $\varepsilon_L(f),\varepsilon_L(g)$ may be estimated by respective empirical estimates, by the usual estimator
$$\widehat{\varepsilon}_L(f) := \frac{1}{M}\sum_{i=1}^M L_i(f)\quad\mbox{where}\; L_i(f) = L(f(X^*_i),Y^*_i),$$
and similar $\widehat{\varepsilon}_L(g)$ for $g$. Since the test data $(X^*_i,Y^*_i)$ are independent of each other and of $f,g$, by the central limit theorem, one has
$$\sqrt{M}(\widehat{\varepsilon}_L(f) - \varepsilon(f))\xrightarrow{d} \mathcal{N}(0, \Var[L(f(X),Y)])\mbox{, }M\rightarrow \infty,$$
conditional on $f$ and $g$ being fixed.
That is, the empirical mean of the loss residuals for prediction strategy $f$ and loss $L$ is asymptotically normal with mean $\varepsilon(f)$ and variance $\Var[L(f(X),Y)])/N$.\\

Instead of directly estimating $\varepsilon_L(f)$ and $\varepsilon_L(g)$ with confidence intervals and then comparing, one notes (as also in~\cite{nadeau2003}) that the samples of loss residuals $L_i(f)$ and $L_i(g)$ are inherently paired, which eventually leads to a more powerful testing procedure.

We hence consider the difference in the $i$-th loss residual, $R_i := L_i(g) - L_i(f)$.
The $R_i$ are still i.i.d.~(conditional on $f,g$) and also have a normal asymptotic, usually with smaller variance than either of $f,g$ in empiry. An effect size of the difference is obtained with normal confidence intervals, and one may conduct a one-sided paired two-sample test for the null
$$H_0: \EE(\varepsilon(g) - \varepsilon(f)) \le 0\mbox{ against } H_A: \EE(\varepsilon(g) - \varepsilon(f)) \lneq 0.$$
Note that the test is one-sided since we test for $f$ to outperform the baseline $g$.\\

To assess this null-hypothesis, two simple tests, one parametric and one non-parametric, are implemented.
\begin{itemize}
\item \textbf{T-test for paired samples \citep{nadeau2003}}: A parametric test assuming that the sample mean of loss residual differences $R_i$ is normally distributed with mean 0 and standard deviation $\sigma$. Under the assumptions of the test and the null hypothesis, the normalized empirical mean $t := \widehat{mu}/\widehat{\sigma}$, where $\widehat{mu}$ and $\widehat{\sigma}$ are the sample mean and standard deviation of the $R_i$, respectively, follows a t-distribution with $M$ degrees of freedom, where we have used the fact that under $H_0$ it holds that $\EE[R] = 0$.

\item \textbf{Wilcoxon signed-rank test}: If the sample mean of loss residual differences $R_i$ is not normally distributed, an increase in power can be achieved by testing if the rank differences are symmetric around the median. Since the $R_i$ in question are differences in loss residuals, there is no ex-ante reason to assume normality, in particular for a ``good'' non-baseline method one may expect that most are positive, hence the normality assumption may be too strong. Hence the non-parametric Wilcoxon class test should be a better general approach to compare loss residuals (see pages~38-52 of~\cite{Corder2009} for the Wilcoxon tests).
\end{itemize}

One subtlety to notice is that instead of testing for the alternative hypothesis ``there exists $f$ such that $\varepsilon_L(f)\lneq \varepsilon_L(\varpi^{(L)}_{Y})$'' which certifies for (in-)dependence, by the above strategy we are testing for the alternative ``there exists $f$ such that for all/the best $g$ it holds that $\varepsilon_L(f)\lneq \varepsilon_L(g)$'' where $g$ is a fixed estimate of $\varpi^{(L)}_{Y}$. However, $g$ is dependent on the training data and itself uncertain. Thus, for a conservative estimate of the significance level, uncertainty in $g$ estimating $\varpi^{(L)}_{Y}$ needs to be taken into account.

In the unconditional case, $\varpi^{(L)}_{Y}$ will be a constant predictor of an elicited statistic of the training labels, such as the mean, median or one minus the majority class frequency, with known central limit theorems that may be used to adapt the Student or Wilcoxon type tests (e.g., by explicitly increasing the variance in the t-test).

In either case, one may also bootstrap the distribution of $g$ and make use of it as follows:
Since the training set is independent of the test set, the function $g$, is as a random variable which is a function of the training data, also independent of the test data. Thus instead of using a fixed $g$ in prediction, one can ensure a conservative test by using a bootstrap sample of pseudo-inputs for the predictions, a different sample of $g$ per test data point that is predicted.

Since the exact form of the correction appears an unresolved sub-problem of the predictive model validation and model selection process in general, and since bootstrapping of $g$ is computationally costly, we have implemented the ``fixed $g$'' variant (potentially liberal) in the package which can be easily adapted to potential corrections based on explicit asymptotics.

We will present an algorithmic implementation in Section~\ref{section: pseudocode} as an algorithmic conditional testing routine, which is later applied to graphical model structure learning, in Section~\ref{sec: gmestimcode}

\subsection{Short note: why permutation of features/labels is not the uninformed baseline}
\label{Sec:badbaseline}

In the light Theorem~\ref{Cor:testing} to use uninformed predictors as baseline for comparison, a much more popular suggestion needs to be discussed: the permutation baseline. We hold that the permutation baseline is practically inappropriate and theoretically inadequate as a baseline, as we explain below.

The permutation baseline is found, in the earliest instance known to us, implemented in the scikit-learn toolbox~\cite{pedregosa2011scikit},
and has recently also been used as a predictive baseline in the specifically relevant context of statistical independence testing, by~\citet{Lopez2017} and~\citet{sen2017model}.

Mathematically, the stylized suggestion is to use as a pseudo-uninformed baseline the predictor $f:\calX\rightarrow \calY$ which has been obtained by learning on modified training data
$(X_{\pi(1)},Y_1),\dots,(X_{\pi(N)},Y_N)$ where $\pi$ is a uniformly sampled random permutation of the numbers $\{1,\dots, N\}$.
The baseline performance is then estimated by the performance of such an $f$, given a sensible choice of learning algorithm.
Alternatively, labels are permuted but not feature vectors, or both are permuted independently, both of which is equivalent to the above as long as the re-sampling for performance estimation is uniformly random. \citet{sen2017model} introduce a more sophisticated bootstrap strategy for conditional independence testing, following the same rationale and a similar reasoning.

The quite strong (enthymemous) argument in favour of using this $f$ as a baseline is that in jumbling the order of the features $X_i$, no $f$ can make sensible use of the $X_i$ because the relation to the $Y_i$ is destroyed in the process.

We argue that this is a non-sequitur and hence a non-argument - the only certain way to prevent prediction of a test label $Y^*$ by a test feature $X^*$ (without preventing prediction altogether) is, by definition of what this means, actually removing access to the test feature $X^*$. Leaving access to $X^*$ may still enable $f$ to predict $Y^*$ well from $X^*$, irrespective of what happened to the training sample.

One argument which actually shows inadequacy of the permutation baseline is that a Theorem such as Theorem~\ref{Thm:uninfind-conditional} cannot be true for a permutation version of ``uninformed'', since a prediction strategy not restricted in its image can always guess a good prediction functional.

Even worse, there are situations where one need not even resort to guessing, where the true prediction functional may even be perfectly recovered from the permuted data by the very same prediction strategy which also works on the unpermuted data. We give a stylized example where this is the case:

Let $(X,Y) := (1,1)\;\mbox{with probability}\; 2/3,\; (-1,-1)\;\mbox{with probability}\; 1/3,$ and consider, as usual, i.i.d.~samples $(X_1,Y_1),\dots, (X_N,Y_N)$ for training.
A naive algorithm which simply assigns the most frequent label to the most frequent feature - such as the one sketched as Algorithm~\ref{alg:countassoc} - will make correct predictions with high probability. This is true no matter whether the permutation is applied or not, and is independent of the size of the dataset, note for example that none of the variables or decisions in Algorithm~\ref{alg:countassoc} change.

\begin{algorithm}[ht]
\caption{Count association classifier, fitting. This is a theoretical counterexample to the permutation baseline, do not use in practice.\newline
\textit{Input:} Training data $\calD = \{(X_1,Y_1),\dots (X_N,Y_N)\}$, where $X_i$ and $Y_i$ take values in $\{-1,1\}$.
both trainable\newline
\textit{Output:} a classical predictor $f:\calX\rightarrow \calY$ \label{alg:countassoc}}
\begin{algorithmic}[1]
    \State Let $N_X\leftarrow \card\{X_i\;:\;X_i = 1\}$
    \State Let $N_Y\leftarrow \card\{Y_i\;:\;Y_i = 1\}$
    \State If $|N_X - N_Y| < |N- N_X + N_Y|$ output $f: x\mapsto x$
    \State Else output $f: x\mapsto -x$
\end{algorithmic}
\end{algorithm}

The counterexample may seem pathological, but in fact it is only stylized. Much less pathological examples may be easily constructed through replacing the numbers $-1,1$ well-separable clusters in whatever continuous or mixed variable space.

A heuristic remedy of the problem is to strongly restrict the class of potential predictors $f$ used as a permutation baseline, or make restrictive regularity assumptions on the joint distribution of $(X,Y)$ and/or the conditioning variable, such as in the recent work of~ \citet{sen2017model}.

However, if one does not wish to make such assumptions, with Theorem~\ref{Thm:uninfind-conditional} and the above examples and considerations, one may conclude that any attempt to define ``uninformedness'' through the data and not through properties of the learning algorithm may be theoretically insufficient, especially since it does not define a valid class of baselines without further modifications or assumptions, as the above examples and considerations show.

In contrast to the permutation baseline criticized above, the ``constant prediction'' definition of uninformedness forbids $f$, by construction, to make use of the relevant information in prediction. The claim that a constant $f$ is indeed prevented from using any learnt information about the prediction rule is mathematically proven in Theorem~\ref{Thm:uninfind-conditional}. 

%% file: 4_graphical-models.tex
\newpage
\section{Graphical Models}\label{sec:gm}
This section will briefly review graphical models and graphical model structure learning, which we will later address using the conditional independence test outlined in Section~\ref{sec:depend}.

\subsection{Informal definition of graphical models}
A probabilistic graphical model is a graph-based description of properties of a joint probability distribution (possibly but not necessarily a full specification thereof).
It usually consists of a directed or undirected graph, together with a fixed convention how to translate the graph structure into a full probability distribution, or more generally, into conditional independence statements about a distribution of not further specified type.

\subsubsection*{Graphs}
A graph $G$ is an ordered pair $(V,E)$ where $V\subseteq E\times E$. The set $V$ are interpreted as vertices (nodes) of the graph, and the elements in $E$ are interpreted as edges (links) of the graph (ordered for directed graphs, unordered for undirected graphs). $G$ is usually visualized by drawing all vertices $V$ and edges between all pairs of vertices in $E$. In graphical models, the vertex set is identified with a collection of (possibly associated) random variables $X = [X_1, ..., X_n]$, and the edges encode some independence or conditional independence information about the components of $X$. Two popular choices are discussed in Section~\ref{sec:indepstatement} below.

\subsubsection*{Probabilistic graphical models}
Graphical model theory is usually concenred two main tasks, structure learning (of the graph) and inference (on a given graph). Inference is concerned with estimating parameters or statistics of the parametric distribution assuming the independence structure prescribed by a fixed graph, from a finite sample drawn from $X$. Structure learning is concerned with inferring the graph from such a sample, thus inferring the encoded independence relations. For an extended introduction to graphical models, the reader might refer to~\cite{koller2009} or~\cite{barber2012}.\\
Graphical model structure learning is usually considered the more difficult of the two tasks, due to the combinatorial explosion of possible graphs.
Manual approaches involve an expert encoding domain knowledge in presence or absence of certain edges; automated approaches usually conduct inference based on parametric distributional assumptions combined with selection heuristics~\cite{koller2009}.

\subsection{Types of graphical models}\label{sec:indepstatement}
We review two of the most frequently used types of graphical models.

\subsubsection{Bayesian Networks}
A Bayesian network is a graphical model which states conditional independence assumptions without making parametric distributional assumptions. The conditional independence assumptions are encoded in a directed acyclic graph over a set of random variables. Acyclicity of the graph implies that each node $X_i$ has a (potentially empty) set of descendants,

\begin{Def}
A node $X_j$ is a descendant of $X_i$ in the graph $G$ if there is a directed path from $X_i$ to $X_j$, where a path is any connection of links that lead from $X_i$ and $X_j$.
\end{Def}

That is, if there is a path following the arrows in $G$, going from $X_i$ to $X_j$. Since the graph is acyclic, no cycles exist, and following the arrows in the graph, the same variable can never be reached twice. Examples of such a networks are shown in Figure \ref{fig:expertgm}.\\

\begin{figure}
  \centering
    \begin{subfigure}[t]{0.15\textwidth}
        \tikz{ %
	    \node[latent, minimum size = 40pt, font=\fontsize{8}{8}\selectfont] (grade) {Grade};
	    \node[latent, below=of grade, minimum size = 40pt, font=\fontsize{8}{8}\selectfont] (letter) {Letter};
	    \edge {grade} {letter};
        }
	\caption{}
    \end{subfigure}%
    \begin{subfigure}[t]{0.3\textwidth}
        \tikz{ %
	    \node[latent, minimum size = 40pt, font=\fontsize{8}{8}\selectfont] (grade) {Grade};
	    \node[latent, below=of grade, minimum size = 40pt, font=\fontsize{8}{8}\selectfont] (letter) {Letter};
	    \node[latent, above=of grade, minimum size = 40pt, font=\fontsize{8}{8}\selectfont] (difficulty) {Difficulty};
	    \node[latent, right=of difficulty, minimum size = 40pt, font=\fontsize{8}{8}\selectfont] (intel) {Intelligence};	
	    \edge {difficulty} {grade};
	    \edge {intel} {grade};
	    \edge {grade} {letter};
        }
	\caption{}
    \end{subfigure}%
    \begin{subfigure}[t]{0.3\textwidth}
        \tikz{ %
	    \node[latent, minimum size = 40pt, font=\fontsize{8}{8}\selectfont] (grade) {Grade};
	    \node[latent, below=of grade, minimum size = 40pt, font=\fontsize{8}{8}\selectfont] (letter) {Letter};
	    \node[latent, above=of grade, minimum size = 40pt, font=\fontsize{8}{8}\selectfont] (difficulty) {Difficulty};
	    \node[latent, right=of difficulty, minimum size = 40pt, font=\fontsize{8}{8}\selectfont] (intel) {Intelligence};	
	    \node[latent, right=of grade, minimum size = 40pt, font=\fontsize{8}{8}\selectfont] (gpa) {GPA};
	    \edge {difficulty} {grade};
	    \edge {intel} {grade};
	    \edge {grade} {letter};
	    \edge {intel} {gpa};
        }
	\caption{}
    \end{subfigure}%
\caption[Expert-based graphical model structure learning]{Example of expert based graphical model structure learning (backtracking), adapted from~\cite{koller2009}.\\(a) The quality of an academic reference letter is determined by the grade \\ (b) The students  intelligence and course difficulty determine the grade.\\ (c) Knowing a students GPA gives additional information about the state of intelligence}
\label{fig:expertgm}
\end{figure}

Bayesian Network graphs arise in a natural way when considering the factorization properties of a probability distribution. Assume we are interested in a probability distribution $P$ over the $Difficulty$ of a course, a students $Intelligence$, the $Grade$ a student achieved in a course, and the quality of a reference $Letter$ received from the tutor of the course, $P(Difficulty, Intelligence, Grade, Letter)$. Further assume the following is true
\begin{itemize}
\item The quality of the $Letter$ is solely determined by the $Grade$ a student received in the course. That is, $Letter \indep \{Difficulty, Intelligence\} | Grade$
\item The $Grade$ of a student depends on the $Difficulty$ of the course and his $Intelligence$
\item $Difficulty$ and $Intelligence$ are not causally influenced by any other variable in the graph
\end{itemize}
This gives a natural way to order the variables for factorization. $Difficulty$ and $Intelligence$ are so-called root nodes, hence their order is irrelevant, however $Grade$ depends on both of them and $Letter$ depends on $Grade$, giving the ordering: $\{Letter, Grade, Difficulty, Intelligence\}$, which will now be denoted as $\{L, G, D, I\}$. An ordering $\{X_1,...,X_n\}$ implies that for $\forall i,j \in \{1,...,n\}$ and $i < j$, $X_j$ can not be a descendant of $X_i$. The distribution is then factorized, according to the ordering,
$$P(L,G,D,I) = P(L|G,D,I)P(G,|D,I)P(D|I)P(I),$$
a standard way to factorize distributions. Working in the independence statements gives
$$P(L,G,D,I) = P(L|G)P(G,|D,I)P(D)P(I),$$
since $L\indep \{D,I\} | G$ and $D \indep I$. Returning to Figure \ref{fig:expertgm} (b) shows that this factorized distribution exactly matches the arrows in the graph, which start at the parents (the conditioning variables) and lead to the children (the variable that the conditional distribution is over).

To understand the independence statements encoded in a Bayesian Network, one needs to first distinguish between active and passive trails \cite[p. 71]{koller2009}. Let a trail be any path between $X_i$ and $X_j$ on $\mathcal{G}$ and a v-structure a structure such that $X_{i-1} \rightarrow X_i \leftarrow X_{i + 1}$, where $X_i$ is a descendant of both $X_{i-1}$ and $X_{i+1}$.
\begin{Def}
A trail between $X_i$ and $X_j$ is active given a conditioning set $Z$, if for every v-structure $X_{i-1} \rightarrow X_i \leftarrow X_{i + 1}$ along the trail, $X_i$ or one if it's descendants is in $Z$ and no other variables on the trail are in $Z$.
\end{Def}
In Figure \ref{fig:expertgm} (c) that means, the trail from \textit{Letter} to \textit{GPA} is active only if conditioned on either the empty set or \textit{Difficulty}, and the trail between \textit{Intelligence} and \textit{Difficulty} is only active if \textit{Grade} or \textit{Letter} is in the conditioning set.
\begin{Def}\cite{barber2012}
If there is no active trail in the graph $G$ between the nodes $X$ and $Y$ given $Z$, then $X\indep Y|Z$ in any distribution consistent with the graph $G$
\end{Def}

\subsubsection{Markov Networks}
A Markov Network, or Markov Random Field, is a graphical model which encodes conditional independence statements in an undirected graph over a set of random variables. While the Bayesian Network is defined in terms of probability distributions, the Markov Network is usually specified in terms of factor products which are unnormalized probability distributions. The scopes of these factors determine the complexity of the system, if the scope of a factor covers all the variables in the graph, the graph is unconstrained, whereas any constraint to smaller scopes decreases its complexity.\\
Figure \ref{fig:markovnet} shows the Markov Network for the (normalized) factor product
$$P(a,b,c,d,e,f,g) = \frac{1}{Z} \phi_1(a,b,c)\phi_2(b,c,d)\phi_3(c,e)\phi_4(e,f)\phi_4(f,g)\phi_5(e,g),$$
with the normalizing constant (partition function) Z. Small letters denote realizations of the capital-lettered nodes. If $P$ is a probability mass function,
$$Z =\sum_{a,b,c,d,e,f,g} \phi_1(a,b,c)\phi_2(b,c,d)\phi_3(c,e)\phi_4(e,f)\phi_4(f,g)\phi_5(e,g),$$
where $\phi(s)$ is such that $\phi:S \rightarrow \mathbb{R}$, $\forall s \in S$.

If $P$ is a probability density function over continuous variables, Z would be attained by integrating over the support of the variables in the scopes of $P$.

\begin{figure}
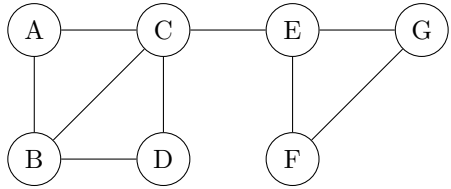

  \centering
        \tikz{ %
	    \node[latent] (A) {A};
	    \node[latent, below=of A] (B) {B};
	    \node[latent, right= of A] (C) {C};
	    \node[latent, below=of C] (D) {D};
	    \node[latent, right= of C] (E) {E};
	    \node[latent, below=of E] (F) {F};
	    \node[latent, right= of E] (G) {G};
	    \edge [-] {A} {B};
	    \edge [-] {A} {C};
	    \edge [-] {C} {B};
	    \edge [-] {D} {B};
	    \edge [-] {C} {D};
	    \edge [-] {C} {E};
	    \edge [-] {E} {F};
	    \edge [-] {E} {G};
	    \edge [-] {F} {G};
        }
\caption{Markov Network}
\label{fig:markovnet}
\end{figure}

Since the links are undirected, independence statements arising from Markov Networks are symmetric. A path between $X_i$ and $X_j$, $i \neq j$ is called active given $Z$, if no node on the path is in $Z$. To determine whether $X_i$ and $X_j$ are conditionally independent given $Z$ in all distributions consistent with $\mathcal{G}$, one again considers all paths going from $X_i$ to $X_j$. It holds that $X_i \indep X_j | Z$ if there is no path between $X_i$ and $X_j$ that is active given $Z$.\\
So to attain a Markov Network that is consistent with a probability distribution, one has to consider the pairwise conditional independence properties of the variables that the nodes in the graph represent. In general, if in a Markov Network there is no link between a variable $X_i$ and $X_j$, then $X_i \indep X_j | X \setminus \{X_i,X_j\}$ in any distribution $P$ over $X$ consistent with the graph. These are called the pairwise Markov-independencies of the graph \cite[ch. 17.2]{friedman2001}.\\[0.1in]
The independencies that can be encoded using Bayesian and Markov Networks differ. Bayesian Networks are natural to express (hypothesized) causal relationships, distributions where an ordering can be attained and thus have nice factorization properties, while Markov Networks are natural for expressing a set of pairwise Markov independencies. Additionally, Bayesian Networks can be transformed into an undirected graph by a (potentially forgetful) process called "moralization" \cite[p. 135]{koller2009}.

\subsection{Graphical model structure learning}
There are two dominant approaches to structure learning, independence testing based and score-based methods, however they both suffer to a varying extent from the same underlying problem: combinatorial explosion. Given a probability distribution $P$ (with associated graph $G$) over a multivariate random variable $X = [X_1, ..., X_n]$, for undirected networks, between any pair of variables $X_i$ and $X_j$ there can be a link or no link in $G$. Since the links are undirected, there are thus $\frac{p(p-1)}{2}$ potential edges in the graph, and an exponential number of possible graphs. If one wants to test if $X_i$ and $X_j$ are independent given $Z$, $Z \subset X$, one needs to test if any path between the two is active, leading to a potentially very large number of tests. A graph's complexity can be decreased by upper bounding the number of edges for each node, however, it was shown that, for directed graphs where each node can have $d$ parents, for $d > 1$, the problem a finding an optimal graph is NP-hard \cite[p. 811]{koller2009}. Generally, the space of graphs is reduced by making structural (what type of relationships can be captured in the graph) and parametric (constraints on the distribution) assumptions. An additional problem is posed when there exists more than one optimum to the structure learning method, as a result of the fact that different graphs can be equivalent (and thus not identifiable) in the respective search space. An example of this are the two graphs shown in Figure \ref{fig:dagequiv}, which arises from the fact that if $\{X,Y\}\sim P$, $P$ can be decomposed either into $P(X|Y)P(Y)$ or $P(Y|X)P(X)$, which are equivalent.
\begin{figure}
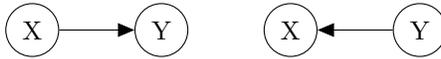

  \centering
        \tikz{ %
	    \node[latent] (X1) {X};
	    \node[latent, right=of X1] (Y1) {Y};
	    \node[latent, right= of Y1] (X2) {X};
	    \node[latent, right=of X2] (Y2) {Y};

	    \edge {X1} {Y1};
	    \edge {Y2} {X2};
        }
\caption[Equivalent DAGs]{Two DAGs equivalent with respect to independence statements}
\label{fig:dagequiv}
\end{figure}
There are different ways to approach these problems, some of which are outlined below.

\subsubsection{Score-based methods}
Score-based methods attempt to find a (usually local) optimum of a score function in the space of possible graphs. Examples of this can be the Kullback-Leibler divergence scoring function for directed graphical models, which can be used to find the optimal Chow-Liu tree graph \cite[p.219]{barber2012}. While this is a convex algorithm, finding the global optimum, it does so under the heavy constraint that each node only has one parent at maximum (resulting in a tree-structured graph). For undirected graphs over discrete variables, gradient descent can be used directly on the log likelihood to find the local optimum \cite[p. 221]{barber2012}, however each gradient evaluation requires calculation of the partition function $Z$ (by summing over all discrete states), which makes this algorithm very expensive. Many of the algorithms for score-based methods are designed for discrete variables, and not applicable when the variables are continuous. Performance is usually evaluated by calculating the likelihood score on a hold-out sample.
One other area not mentioned above is \textit{Bayesian model averaging} (e.g. \cite{Fragoso2015}), which seeks to improve the above methods by averaging over different model structures. It can be viewed as an ensembling method for the score-based methods mentioned above. State-of-the-art score based methods are oftentimes very expensive or make very strong assumptions on the underlying distribution $P$, such as tree-structure of the graph \citep[ch. 18]{koller2009}.

\subsubsection{Independence testing based methods}\label{sec:indepbased}
Unlike the score-based models, independence testing based models conduct independence tests between variables or sets of variables locally and then aggregate the information to produce the graphical model. These sets of independence tests are then used to infer the structure of the graph in a straightforward manner, based on the properties of the respective type of graphical model. An example of this is the PC-Algorithm~\cite[e.g. p. 214]{barber2012} which is used in attempts at causal discovery. Another algorithm for recovering the undirected skeleton is given by Algorithm 3.3 of~\cite{koller2009}. 

There are probably two main reasons that this latter independence testing based method is not used for graphical modelling:
\begin{itemize}
\item[(a)] is that it relies on a conditional independence test where the conditioning is on multiple variables, i.e., all except two. As outlined in Section~\ref{sec:background}, this is a largely unsolved problem expect when there is at most one conditioning variable, i.e., if there are three or less variables in total.
\item[(b)] It is hard to regulate the bias-variance trade-off, unlike for score-based methods where this may be achieved through the use of constraints such as an upper bound on the number of parents.
\end{itemize}

However, these two issues may be addressed through our novel predictive conditional independence testing framework:
\begin{itemize}
    \item[(a)] The predictive conditional independence test described in the subsequent Section~\ref{section: pseudocode}, based on the theoretical insights in Section~\ref{sec:depend}, allows for efficient conditional independence testing with variables of any dimension, and thus provides a predictive approach to learning graphical models.
    \item[(b)] The link to the supervised learning workflow allows for direct adoption of well-known strategies to trade-off bias and variance and error estimation from supervised learning, such as regularization and cross-validatied tuning, to the graphical modelling workflow..
\end{itemize}

In addition, our suggested algorithms have further desirable features:
\begin{itemize}
\item The intimate link to supervised learning also allows for a controlled trade-off between power of the algorithm and time complexity. This trade-off is readily available when using the conditional independence testing routine described in Section ~\ref{section: pseudocode}, since the user can choose estimation methods of varying power/complexity in order to influence the total run time needed by the algorithm.
\item Section~\ref{sec:fdr} will also introduce a false-discovery-rate control routine to provide a tool that lets a user control the proportion of false-positives (erroneously found edges in the estimated  graphical model structure), regardless of the size of the graph.
\end{itemize}

\subsubsection*{Note on causal inference}
Graphical models, and specifically Bayesian Networks, are a natural way to express hypothesized causality, since the arrows seem to express causation. However, when actually learning graphical models from data, causality may only be inferred by making strong (and often incorrect) assumptions on the underlying distribution, or by collecting data in a manner that allows for causal interpretation, namely in proper experimental set-ups (e.g., following Pearl's do-calculus, or standard experimental study design). As generally for graphical model structure learning, all algorithms outlined in this paper are not to be interpreted as expressing causality, but merely as producing a collection of statements about association which certify for causality only in combination with the right experimental set-up.

\newpage 

%% file: 5_algorithms.tex
\section{Predictive inference algorithms}\label{sec:API}
This section will first introduce the proposed predictive conditional independence testing routine (PCIT), which is based on the methodology outlined in Section \ref{sec:depend}, and important subroutines related to performance and error control of the test. After, an algorithm to leverage the independence test into a graphical model structure learning routine that addresses the issues outlined in \ref{sec:indepbased} is presented.

\subsection{Predictive conditional independence testing}\label{section: pseudocode}
Algorithm \ref{alg:indeptest} implements the results from Section \ref{sec:depend} to test if a set of variables $Y$ is independent of another set of variables $X$, given a conditioning set $Z$ (optional, if not provided the algorithm tests for marginal independence). It will later be used as a subroutine in the graphical model structure learning, but can also be used in it's own right for tasks such as determining if a subset $X$ would add additional information when trying to predict $Y$ from $Z$.

\begin{algorithm}[h!]
  \caption{Predictive conditional independent test (PCIT)}
    \label{alg:indeptest}
  \begin{algorithmic}[1]
     \\Split data into training and test set
    \For{all variables y $\in$ Y}
\\ \hspace{0.1in} \textbf{on training data:}
\\ \hspace{0.18in} find optimal functional $f$ for predicting y from Z
\\ \hspace{0.18in} find optimal functional $g$ for predicting y from \{X,Z\}
\\
\\ \hspace{0.1in} \textbf{on test data:}
\\ \hspace{0.18in}calculate and store p-value for test that generalization loss of $g$ is lower than $f$
\\
\EndFor
\\
\If{symmetric test needed}
\\  \hspace{0.1in} exchange X and Y, repeat above process
\EndIf
\\
\Let{p\_values\_adjusted}{Apply FDR control to array of all calculated p-values}
\\
\Return p\_values\_adjusted
\end{algorithmic}
\end{algorithm}
When the test is used as a marginal independence test, the optimal prediction functional for $f$ is the functional elicited by the loss function, that is, the mean of $y$ for continuous outputs, and the class probabilities of $y$ for the discrete case (Appendix \ref{app: elicit}). The link to supervised learning further allows/forces the user to distinguish between two cases. Independence statements are symmetric, if $X$ is independent of $Y$, then $Y$ is independent of $X$, in both the marginal and conditional setting. The same cannot be said in supervised learning, where adding $X$ to predicting $Y$ from $Z$ might result in a significant improvement, but adding $Y$ does not significantly improve the prediction of $X$, given $Z$ (as can be seen in the asymmetry of OLS). So if a user is interested in a one-sided statement, the algorithm can be run for a single direction, for example to evaluate if a new set of variables improves a prediction method, and is thus worth collecting for the whole population. If one is interested in making a general statement about independence, the test is ``symmetrized" be exchanging $X$ and $Y$, and thus testing in both directions, and FDR control be applied to the union of the p-values.\\
It is important to distinguish between the two types of tests and null-hypotheses in this test. On one hand (in the symmetric case) for each variable in $X$ and $Y$, it will be assessed, if adding $X$ to the prediction of $y \in Y$ from $Z$ results in an improvement (and vice versa). The null-hypothesis of this ``prediction-null'' is that no improvement is achieved. After all these p-values are collected, a multiple testing adjustment is applied (Section \ref{sec:fdr}), after which the original null-hypothesis, that $X$ and $Y$ are conditionally independent, is assessed. We reject this ``independence-null'', if any one of the ``prediction-nulls'' can be rejected after adjustment. The p-value of the ``independence-null'' hence reduces to the lowest p-value in all the ``prediction-nulls''.
As such, the null of the independence test is that all the ``prediction-nulls'' are true. False discovery rate control, the chosen multiple testing adjustment, is appropriate since it controls the family-wise error rate (FWER), the probability of making at least one type 1 error, if all the null-hypotheses are true \citep{benjamini1995}.

\subsubsection{False-discovery rate control}\label{sec:fdr}
To account for the multiple testing problem in Algorithm \ref{alg:indeptest}, the \textit{Benjamini - Hochberg - Yekutieli procedure} for false-discovery rate (FDR) control is implemented \citep{Benjamini2001}. In their paper, they state that traditional multiple testing adjustments, such as the Bonferroni method, focus on preserving the FWER. That is, they aim to preserve the probability of making any one type 1 error at the chosen confidence level. As a result, tests are usually very conservative, since in many multiple-testing scenarios the p-values are not independently distributed (under the null), and thus the power of these tests can be significantly reduced.\\

In their 2001 paper, they propose to control the false discovery rate instead, \textit{``[..] the expected proportion of erroneous rejections among all rejections''}, as an alternative to the FWER. The FDR allows for more errors (in absolute terms) when many null-hypothesis are rejected, and less errors when few null-hypotheses are rejected.
\begin{algorithm}
  \begin{algorithmic}[1]
    \Require{Set $\{p_{(i)}\}_{i=1}^m$ s.t. $p_j$: p-value for observing $X_j$ under $H_0^j$
    }
    \\
	Sort the p-values in ascending order, $p_{(1)} \leq ... \leq p_{(m)}$ \\
	Let $q = \alpha / (\sum_{i=1}^m{1 / i})$ for chosen confidence level $\alpha$ \\
	Find the \textbf{k} s.t. $k = max(i:p_{(i)}\leq \frac{i}{m}q)$ \\
	Reject $H_0^j$ for $j \in {1,...,k}$
  \end{algorithmic}
 \caption{The Benjamini-Hochberg-Yekuteli Procedure
 \label{alg:FDRcontrol}}
\end{algorithm}
\\ Algorithm \ref{alg:FDRcontrol} shows the procedure outlined in \cite{Benjamini2001}. They showed, that this procedure always controls the FDR at a level that is proportional to the fraction of true hypotheses. As hinted at before, while this algorithm controls the false-discovery rate, in the special case where all null-hypotheses in the multiple testing task are assumed to be true, it controls the FWER, which then coincides with the FDR. This is especially useful, since both scenarios occur in the graphical model structure learning routine described in Section \ref{sec: gmestimcode}.\\
For the choice of optimal false discovery-rate for an investigation, even more so than in the classical choice of appropriate type 1 error, there is no simple answer for which rate might serve as a good default, and it is highly application dependent. If the goal of the procedure is to gain some insight into the data (without dire consequences for a false-discovery), a user might choose a FDR as high as 20\%, meaning that, in the case of graphical model structure learning, one in five of the discovered links is wrong on average, which might still be justifiable when trying to gain insight into clustering properties of the data. This paper will still set the default rate to 5\%, but deviate willingly from this standard whenever deemed appropriate, as should any user of the test.

\subsubsection{Improving the prediction functionals}\label{sec:ensemble}
In practice, when assessing the individual ``prediction-nulls'' in Algorithm \ref{alg:indeptest}, the power of the test (when holding the loss function constant) depends on the capability of the the prediction method to find a suitable functional $g$ that outperforms the baseline $f$. That means, a general implementation of the independence test needs to include a routine to automatically determine good prediction functionals for $g$ and $f$. The implementations in the pcit package presented in Section \ref{sec:package} support two methods for the automatic selection of an optimal prediction functional. Both methods ensemble over a set of estimators, which are shown in Table \ref{table: ensemble}. The prediction functionals refer to the estimator names in sklearn\footnote{Details can be found here \url{http://scikit-learn.org/stable/modules/classes.html}}. Some are constrained to specific cases (e.g. \textit{BernouilliNB}, which only applies when the classification problem is binary).
\begin{table}
\centering
\begin{tabular}{l | l l}
 & Regression & Classification \\ \hline
 Stage 1 & ElasticNetCV & BernoulliNB \\
  & GradientBoostingRegressor & MultinomialNB \\
  & RandomForestRegressor & GaussianNB \\
  & SVR & SGDClassifier\\
  & & RandomForestClassifier \\
    & & SVC\\\hline
 Stage 2 & LinearRegression&LogisticRegression
\end{tabular}
\caption{Prediction functionals used for Stacking/Multiplexing}
\label{table: ensemble}
\end{table}

\paragraph{Stacking}Stacking refers to a two-stage model, where in the first stage, a set of prediction function is fit on a training set. In the second stage, another prediction function is fit on the outputs of the first stage. If the prediction function in the second stage is a linear regression, this can be viewed as a simple weighted average of the outputs in the first stage. In theory, stacking allows the user to fit one single stacking predictor instead of having to compare many potential prediction functions based on the model diagnostics, as in the second stage, better methods get more weight (in the expectation). As an additional effect, improvement in the prediction accuracy through ensembling of predictors can take place (see e.g. Section 19.5 of \cite{Aggarwal2014}). The used stacking regressor and classifier can be found in the Python package \textit{Mlxtend} \citep{mlxtend}.

\paragraph{Multiplexing}When multiplexing over estimators, the training set is first split into a training and validation set, a common procedure to find optimal hyperparameters. After, the predictors are fit individually on the training set, and each predictors expected generalization loss is estimated on the validation set. One then proceeds by choosing the predictor with the lowest estimate for the empirical generalization loss, and refits it using the whole training data (including the former validation set), and then uses the fitted estimator for prediction.

\subsubsection{Supervised learning for independence testing}
Algorithm \ref{alg:indeptest} shows how to leverage supervised learning methodology into a conditional independence test. This has a major advantage over the methods outlined in Section \ref{sec:background}, as the supervised prediction workflow is of utmost interest to many areas of science and business, and, as a result, a lot of resources are going into development and improvement of the existing methods. By making a link between predictive modelling and independence testing, the power of independence testing will grow in the continuously increasing power of the predictive modelling algorithms.

\subsection{Predictive structure learning of undirected graphical models}\label{sec: gmestimcode}
This section outlines a routine to learn the vertices in a directed graph (the skeleton) for a data set by conducting a range of conditional independence tests with the null hypothesis of conditional independence. Section \ref{sec:indepstatement} outlines the conditional independence statements of a Markov network. In a directed graph, if variables $X_i$ and $X_j$ have no direct edge between them, they are conditionally independent given all other variables in the network.
\begin{algorithm}
  \caption{Undirected graph structure estimation}
    \label{alg:undirstruct}
  \begin{algorithmic}[1]
    \For{any combination $X_i$, $X_j$ s.t. $i \neq j$}
    \Let{$X_-$}{$X \setminus \{X_i,X_j\}$}
    \Let{$p\_val_{i,j}$}{p-value for test $X_i \indep X_j |X_-$}
    \EndFor
    \Let{p\_val\_adj}{Apply FDR control on p\_val matrix}
    \\
\Return p\_val\_adj
\end{algorithmic}
\end{algorithm}

Algorithm \ref{alg:undirstruct} describes the skeleton learning algorithm for an input set $X = [X_1, ..., X_n]$, by considering all possible pairs of variables in the data set, and testing if they are conditionally independent given all other variables. The output p\_val\_adj is a symmetric matrix with entries i,j being the p-value for the hypothesis that in the underlying distribution, $X_i$ and $X_j$ are independent, given all other variables in $X$, and hence in the graph $G$ describing it, there is no link between the vertices for $X_i$ and $X_j$. Ultimately, links should be drawn where the adjusted p-values are below a predefined threshold. There are O($n^2$) function evaluations in the for-loop, where $n$ is the number of variables. Section~\ref{sec:experiments} will provide experimental performance statistics for the algorithm and showcase applications on real data sets.

%% file: 6_workflow.tex
\section{pcit package}\label{sec:package}
\subsection{Overview}
The Python\footnote{\url{https://www.python.org/}} package implementing the findings and algorithms of this paper can be found on \url{https://github.com/alan-turing-institute/pcit} and is distributed under the name \href{https://pypi.python.org/pypi/PCIT}{pcit} in the Python Package Index. This section will first provide an overview of the structure and important functions of the package. As this implementation can be thought of as a wrapper for scikit-learn (sklearn) estimators, this section will then describe the sklearn interface and how the package interacts with the sklearn estimators. Lastly, simple application-examples are given.

\subsubsection{Use cases}
The package has two main purposes, independence testing and structure learning. While univariate unconditional independence testing is possible, it is not expected to outperform current methodology in the simple tasks. The main use cases are:

\paragraph{Multivariate independence tests,}\hspace{-0.09in} such as for checking whether there is association between two sets of variables. For example, testing association between demographics or customer behaviour; or, for hedging purposes to determine which financial products are independent from the ones already in the portfolio (under the assumption of independent samples which is not always true in this setting).
\paragraph{Conditional independence tests,}\hspace{-0.09in} such as for assessing association while controlling (= conditioning on) other variables. For example, testing an intervention has an effect while controlling for observational variables; or, testing whether new, costly measurements add predictive/associative power over a set of easy-to-obtain measurements.
\paragraph{Graphical model structure learning,}\hspace{-0.09in}such as for finding clusters in the data as part of an exploratory data analysis, for thoroughly investigating associations in the data, or causal relations in the presence of an intervention (both in section \ref{sec:econ}).\\[0.1in]
For these tasks, the PCIT package serves as a readily available tool that works without the need for manual choices or hyperparameter tuning, and scales well in the dimensionality of the data.

\subsubsection{Dependencies}
The package has the following dependencies:
\paragraph{Scipy \citep{Scipy},}\hspace{-0.09in}for the calculation of p-values
\paragraph{Sklearn \citep{scikit-learn},}\hspace{-0.09in}for its estimators (predictors)
\paragraph{Mlxtend \citep{mlxtend},}\hspace{-0.09in}for the implementation of stacking

\subsection{API description}\label{sec:APInew}
The package introduces three main routines, one for automated prediction (MetaEstimator), conditional independence testing (PCIT), and undirected graph structure learning (find\_neighbours). The following section gives on overview, the function signatures can be found in Appendix \ref{app:signatures}.

\subsubsection*{MetaEstimator}
The MetaEstimator class provides a user with a type-independent predictor that automates model selection for given training data, by automatically determining appropriate loss functions and prediction functionals. It is initialized for sensible defaults, which should generally lead to good results, but can be changed for specific tasks (such as the use of more complex base estimators for more powerful, but also more computationally demanding, predictions). The ensembling methods used for the estimator are described in Section \ref{sec:ensemble}. For regression, the square loss is used for training and calculation of the residuals, for classification, the logistic loss serves as the loss function. \\

\subsubsection*{PCIT}
PCIT implements the conditional independence test in Algorithm \ref{alg:indeptest} to test if two samples stem from conditionally (or marginally) independent random variables. The MetaEstimator class is used as a prediction functional, and hence the user can trade off between computational complexity and power by adjusting the chosen MetaEstimator. That is, if speed is important, the used MetaEstimator should be a combination of base estimators and ensembling method that is quick in execution, whereas if computational resources are vast and a more powerful test is needed, the used base-estimators should be highly tuned.\\

\subsubsection*{find\_neighbours}
find\_neighbours implements Algorithm \ref{sec: gmestimcode} to learn the undirected skeleton for an input data set $X$, using the PCIT.

\subsection{Function signatures}

\subsubsection*{PCIT}
Type: function\\[0.1in]
\begin{tabular}{>{\small}l >{\small}l >{\small}l >{\small}l} \hline
& Name &Description (type)& Default \\ \hline
Inputs:  & x & Input data set 1 ([n x p] numpy array) & \\
   & y & Input data set 2 ([n x q] numpy array) & \\
 & z & Conditioning set ([n x r] numpy array) & None (empty set) \\
 & estimator & Estimator object to use for test (MetaEstimator) & MetaEstimator()\\
  & parametric & Parametric or nonparametric test (bool), section \ref{section: signiftest} & False \\
  & confidence & Confidence level for test (float [0,1]) & 0.05 \\
  & symmetric & Conducts symmetric test (bool), section \ref{section: pseudocode} & True \\
 \hline
 Outputs: & p\_values\_adj &  p-values for "prediction nulls" of each $y\in Y$  (list) & \\
 & independent & tuple, first value shows if "independence-null" is rejected (bool) & \\
 & & \ \ second value is p-value of "independence-null"  (float [0,1]) & \\
  & loss\_statistics & RMSE difference for baseline $f$ and altern. $g$, & \\
  & & loss residuals with standard deviation, for each $y \in Y$. & \\
  & & Only applicable if Y continuous & \\
\end{tabular}

\paragraph{Note:}\hspace{-0.09in}The variance of the difference is estimated by assuming 0 covariance between the residuals of baseline $f$ and alternative $g$, which generally leads to more conservative confidence intervals for error residuals (due to the irreducible error, prediction residuals for different methods are generally positively correlated).

\newpage

\subsubsection*{MetaEstimator}\label{app:signatures}
Type: class
\subsubsection*{Methods}

\begin{tabular}{>{\small}l >{\small}l >{\small}l >{\small}l} \hline \hline
\textbf{init} & Name &Description (type)& Default \\ \hline
 Inputs:  & method & Ensembling method ('stacking', 'multiplexing' or None) & 'stacking'  \\
    & estimators & Estimators to ensemble over (2-tuple of lists of sklearn & None (default estim, \\
    && \ estimators [regression estim], [classification estim]) & \ section \ref{sec:ensemble})\\
    & method\_type & Task definition ('regr', 'classif', None) & None (auto selection) \\
    & cutoff & Cutoff for automatic selection of method\_type (integer) & 10\\ \hline \hline

\textbf{get\_estim} & & &\\ \hline
 Inputs:  & y & Dependent variable ([n x 1] numpy array) & \\
  Outputs: & estimators & Appropriate set of estimators (list) & \\ \hline \hline

\textbf{fit} & & &\\ \hline
 Inputs:  & x & Independent variables ([n x p] numpy array) & \\
  & y & Dependent variable ([n x 1] numpy array) & \\
  Outputs: & fitted & Fitted estimator (MetaEstimator) & \\ \hline \hline

\textbf{fit\_baseline} & & & \\ \hline
 Inputs:  & x & Independent variables ([n x p] numpy array) & \\
  & y & Dependent variable ([n x 1] numpy array) & \\
  Outputs: & fitted & Fitted uninformed baseline estimator (MetaEstimator) & \\ \hline \hline

\textbf{predict} & & & \\ \hline
Requires: &\multicolumn{2}{l}{MetaEstimator has been fitted} & \\
 Inputs:  & x & Test set independent variables ([n x p] numpy array) & \\
  Outputs: & predictions & Predictions for test set ([n x 1] numpy array) & \\ \hline \hline

 \textbf{get\_resid} & & & \\ \hline
 Inputs:  & x\_train & Training set independent var. ([n x p] numpy array) & \\
   & y\_train & Training set dependent variables ([n x 1] numpy array) & \\
   & x\_test & Test set independent variables ([n x p] numpy array) & \\
   & y\_test & Test set dependent variables ([n x 1] numpy array) & \\
   & baseline & Should baseline be fitted (boolean) & False \\
  Outputs: & resid & Residuals for prediction strategy ([n x 1] numpy array) & \\ \hline \hline

\end{tabular}

\subsubsection*{find\_neighbours}
Type: function\\
\begin{tabular}{>{\small}l >{\small}l >{\small}l >{\small}l} \hline
& Name &Description (type)& Default \\ \hline
Inputs:  & X & Input data set ([n x p] numpy array) & \\
  & estimator & Estimator object to use for test (MetaEstimator) & MetaEstimator()\\
  & confidence & False-discovery rate (float [0,1]) & 0.05 \\ \hline
 Outputs: & skeleton & Matrix, p-values for each indep test ([p x p] numpy array) &  \\
 & skeleton\_adj  & Learnt graph after applying FDR control ([p x p] numpy array) & \\
\end{tabular}

\subsection{API design}\label{sec:sklearnwf}
The API is designed as a wrapper for Scikit-learn (sklearn), a package in the Python programming language, that aims to provide a user with a consistent, easy-to-use set of tools to analyze data \citep{scikit-learn}. It is one of the most-used tools in today's supervised learning community, which is why it is chosen as the supervised prediction workflow to build on for the predictive conditional independence test. This section will outline the advantages of basing the test on the sklearn package.

\subsubsection{Sklearn interface}\label{sec:sklearn}
Sklearn is built around estimator objects, which implement a consistent set of methods. Estimators provide a fit and, if applicable, a predict method, in order to be able to fit the estimator to training data and then predict on a test set. Additionally, sklearn provides standardized approaches to model selection and hyperparameter-tuning as well as data transformation and ensembling methods (see \cite{buitinck2013} for a more thorough discussion). The methods can easily be combined to create more powerful estimators. Defaults are chosen sensibly so that in most cases, an initial fit of a method to data requires little manual parameter specification from the user's side. While the highly automated and simplified approach of sklearn lowers the bar of entry when aiming to generate knowledge from data, it also comes with a downside. For most statistical applications that exceed fitting and predicting from a data set, such as inference on the parameters and hypotheses about prediction accuracies, the relevant subroutines are missing from the API. Relevant statistics can however be attained by interfacing it with other packages (such as SciPy).

\subsubsection{Wrapper for Sklearn estimators}
As we saw in section \ref{sec:APInew}, the conditional independence test described in Algorithm \ref{alg:indeptest} uses the newly defined \textbf{MetaEstimator} class to automate determining the optimal prediction functional for a given task. It does so, by ensembling over a set of base estimators from sklearn. These are either chosen to be the sensible defaults described in Table \ref{table: ensemble}, or can be passed by the user as a tuple of lists of sklearn base estimators. This is required since regression and classification tasks rely on vastly different prediction functionals, and thus need to be specified separately. As a general rule, the passed regressors need to have a fit and a predict method, whereas the classifiers need to have a fit and a predict\_proba method. Requirements might be more stringent for certain types of data or certain estimators, however specifying an unsuitable estimator class will result in an error message as specified by the respective class, allowing the user to either remove the unsuitable class or proceed with the sensible defaults. As mentioned before, this gives a user a flexible tool to trade off between power and computational complexity. If in need of a fast method, one can use an algorithm that runs in linear time, such as stochastic gradient descent linear regression, whereas if a test with high power is needed, one can pass hyper-tuned estimators to the function, that take longer to run but generalizes better for prediction on unseen data.

\subsection{Examples}
This section will provide some simple examples of how the code is used. For the following it is assumed that data sets X, Y and Z, all of the size [number of samples $\times$ number of dimensions], are loaded as numpy arrays, and have matching numbers of samples (sample indices in X, Y and Z correspond to the same sample). After installing the pcit package, import the relevant objects:

\begin{lstlisting}[language=Python, backgroundcolor = \color{lightgray}, numbers = left]
from pcit import MetaEstimator, StructureEstimation, IndependenceTest
\end{lstlisting}

Testing if $X \indep Y | Z$, using the default values:

\begin{lstlisting}[language=Python, backgroundcolor = \color{lightgray}]
IndependenceTest.PCIT(X, Y, z = Z)
\end{lstlisting}

Testing if $X \indep Y | Z$, with a custom MetaEstimator, multiplexing over a manually chosen set of estimators:

\begin{lstlisting}[language=Python, backgroundcolor = \color{lightgray}]
from sklearn.linear_model import RidgeCV, LassoCV,
                    SGDClassifier, LogisticRegression

regressors = [RidgeCV(), LassoCV()]
classifiers = [SGDClassifier(), LogisticRegression()]

custom_estim = MetaEstimator.MetaEstimator(method = 'multiplexing',
                                estimators = (regressors, classifiers))

IndependenceTest.PCIT(X, Y, z = Z,
                estimator = custom_estim)
\end{lstlisting}

Learning the undirected skeleton of X:

\begin{lstlisting}[language=Python, backgroundcolor = \color{lightgray}]
StructureEstimation.find_neighbours(X)
\end{lstlisting}
Concrete outputs are shown in section \ref{sec:experiments} below.

%% file: 7_experiments.tex
\section{Experiments}\label{sec:experiments}
This section will first evaluate the performance of the proposed algorithms, and then provide some examples of applications on real world data sets. All performance tests can be found on Github\footnote{\url{https://github.com/alan-turing-institute/pcit/tree/master/tests}}.

\subsection{Performance tests}\label{sec:performance}
This section will report on performance tests for the algorithms derived in Section \ref{sec:API}. First the power of the the predictive conditional independence routine is bench-marked against current state-of-the-art methodology, then various tests on the directed graph structure learning algorithm are conducted.

\subsubsection{Performance of conditional independence test}
In this section the conditional independence routine will be bench-marked against the previous research, namely the kernel based approach for conditional independence testing, which is is shown to be more powerful than other conditional independence testing algorithms in \citep{Zhang2012} (on multivariate Gaussian data). The used code for the kernel test is taken from GitHub\footnote{\url{https://github.com/devinjacobson/prediction/tree/master/correlations/samplecode/KCI-test}}. To conduct a test using data that is drawn from a distribution that more closely resembles real world data, as opposed to the synthetic (Gaussian) data commonly used for performance tests, the UCI Wine Repository data set \citep{lichman2013} is used as follows:
\begin{itemize}
    \item The columns 'Alcohol', 'Malic Acid' and 'Magnesium' are randomly permuted (to make them independent) and will serve as $X$, $Y$ and $noise$ arrays respectively
    \item Vector $Z$ is created by individually sampling vectors $X'$, $Y'$ and $noise'$ of size n with replacement from $X$, $Y$ and the noise vector, and then calculating $$Z_i = \text{log}(X'_i) \times \text{exp}(Y'_i) + u * \sqrt{noise'_i}\mbox{, }i \in \{1,...,n\},$$ where $u$ is the sign, uniformly drawn from $\{-1,1\}$
\end{itemize}

This results in a scenario, where $X \indep Y$, but $X  \not\!\perp\!\!\!\perp Y | Z$, and the signal to noise ratio is high for small sample sizes. The test will be conducted by increasing the sample size from 100 to 5000, and calculating the run times and power for both approaches. For each sample size, the PCIT is run 500 times, and the KCIT is run 200 times, since the KCIT is more computationally demanding than the PCIT. The only exception is $n = 5000$ for the KCIT, which is run 25 times, since the time complexity would be too high to draw more samples. Both methods are run for their default values, without additional manual tuning, and at a confidence level of 5\%. Each time, $\{X',Y',Z\}$ is sampled, and then the conditional independence tests are applied and the run time is recorded. If they reject independence at a 5\% level, the round counts as a success, 1, otherwise 0. The power and run times are then calculated by averaging over the results, and standard errors for the power are attained by realizing that the standard error of the power for a rerun number of B is the standard error of $\frac{X}{B}$, where $X \sim Bin(B,\theta)$, where $\theta$ is the observed power (the sample mean).

\begin{table}
\centering
 \begin{tabular}{l | l | r r r r r r}
  & n & 100 & 200 & 500 & 1000 & 2000 & 5000 \\ \hline
 \multirow{2}{*}{PCIT} & Power & $\underset{(0.006)}{0.020}$ & $\underset{(0.009)}{0.046}$ &$\underset{(0.021)}{0.332}$ &$\underset{(0.021)}{0.672}$ &$\underset{(0.017)}{0.832}$ &$\underset{(0.007)}{0.970}$ \\
 & Time (s) & 0.32 & 0.38 & 0.49 & 0.624 & 1.31 & 4.79 \\ \hline
  \multirow{2}{*}{KCIT} & Power & $\underset{(0.015)}{0.050}$ & $\underset{(0.019)}{0.085}$ &$\underset{(0.027)}{0.185}$ &$\underset{(0.033)}{0.325}$ &$\underset{(0.028)}{0.8}$ & $\underset{(*)}{1}$\\
 & Time (s) & 0.57 & 1.25 & 9.8 & 44 & 383 & 4758 \\ \hline
 & Stand. difference & -1.8 & -1.78 & 4.25 & 8.85 & 0.97 & *\\
 \end{tabular}
 \caption[Power and complexity comparison]{Performance statistics for the newly proposed predictive conditional independence test (PCIT) and the kernel based approach (KCIT). The values in brackets show the estimated standard errors. The last row shows the standardized difference between the power estimates, PCIT - KCIT}.
  \label{table:results}
\end{table}

The results are shown in Table \ref{table:results}. The power at the higher end of the sample sizes seems to be similar (it is important to note that the power of 1 for the KCIT for $n = 5000$ was achieved on 25 resamples only), where as in the range 500 to 1000 samples, the proposed predictive conditional independence test shows a significantly higher power. For small $n$, the KCIT seems to fare significantly better, however both approaches have a very low power, and the PCIT especially shows the power levels below the confidence levels, which might indicate a discrepancy between true type 1 error and expected type 1 error. Important to note is the very high computational complexity of the kernel-based approach for a data set of size 5000, with a run time of approximately 80 minutes per test, while the predictive conditional independence test still has a very low run time of 4.8 seconds. This is to be taken with a grain of salt, since the tests were run in different languages (PCIT in Python, KCIT in MATLAB), but it is apparent that PCIT scales much better than KCIT, while the both converge to a power of 1.

\subsubsection{Performance of structure learning algorithm: Error rates}
This test will show that the graphical model learning algorithm is capable of recovering the true graph with high probability as the sample size increases. No comparison with alternative tests is made, as it would lead to infeasible run times.\\[0.1in]
The data used in the performance tests is generated as follows:
\begin{enumerate}
    \item Sample a random positive definite, sparse, symmetric precision matrix $M$. The size of the entries of this matrix (relative to the diagonal) determine the signal to noise ratio, and are thus thresholded to 0, if below a certain value.
    \item Invert $M$, $M' = inv($M$)$ and use $M'$ as covariance matrix to sample from a multivariate Gaussian distribution. This has the effect, that the zero-entries in $M$ express zero-partial correlations \citep[ch. 17.3]{friedman2001} between respective variables (and hence, lack of an edge between the two nodes in the graph). That is, for a multivariate Gaussian random variable $X$, $X = [X_1, ..., X_p]$, $M_{i,j} = 0 \implies X_i \indep X_j | X \setminus \{X_i, X_j\}$.
    \item Sample $\mathcal{D}$, a data set of size $n$, from the multivariate normal $P = N(0, M')$. The choice of $n$ will allow to evaluate the algorithms performance for an increasing signal to noise ratio.
    \item The undirected graph $\mathcal G$ consistent with the probability distribution $P$ is given by M, where edges occur between variables $X_i$ and $X_j$, if $M_{i,j}$ is non-zero.
\end{enumerate}

Then, Algorithm \ref{alg:undirstruct} will be tested by getting an estimate $\hat{\mathcal{G}}$ of the \textbf{structure of an undirected graph} $\mathcal{G}$ induced by the distribution $P$ from which the data set $\mathcal D$ was drawn. The performance evaluation will be guided by three metrics:
\begin{itemize}
    \item False-discovery rate: The FDR is the fraction of type 1 errors (edges in $\hat{\mathcal{G}}$ that are not in $\mathcal{G}$) over the total number of identified edges in the learned $\hat{\mathcal{G}}$
    \item Power: Number of found edges (links in true graph $\mathcal{G}$ that were found by the structure learning algorithm) over the total number of links in $\mathcal{G}$
    \item Time: Run time needed to estimate $\hat \calG$
\end{itemize}
For the test, the number of random variables is chosen to be 10. This means, that each node in the graph has up to 9 neighbours, and the total number of possible undirected links (before inducing sparsity) is 45. The sparsity parameter is chosen in a way that generally between 10 and 15 of those links exist in the true graph. The size and sparsity of the graph are chosen to produce estimators of the metrics with reasonably low variance, but are otherwise arbitrary. The sample sizes range from approximately 400 to 20000, increasing in steps of 10\%. For each sample size, 10 tests are run to decrease the variance in the estimators. The test is conducted for conditional independence testing using stacking and multiplexing, as well as without using any ensembling method, which, since the data is continuous, results in the usage of Elastic Net regularized regression.\\
If the algorithms work as expected, the FDR is expected to be at or below 5\%. High power-levels indicate a better performance of the algorithm, with respect to this specific task. At the very least, the power level is expected to increase in the number of samples, suggesting that asymptotically the routine will find the correct graph.

\begin{figure}
\centering
    \begin{subfigure}[t]{\textwidth}
        \centering
        \includegraphics[width = 0.7\textwidth]{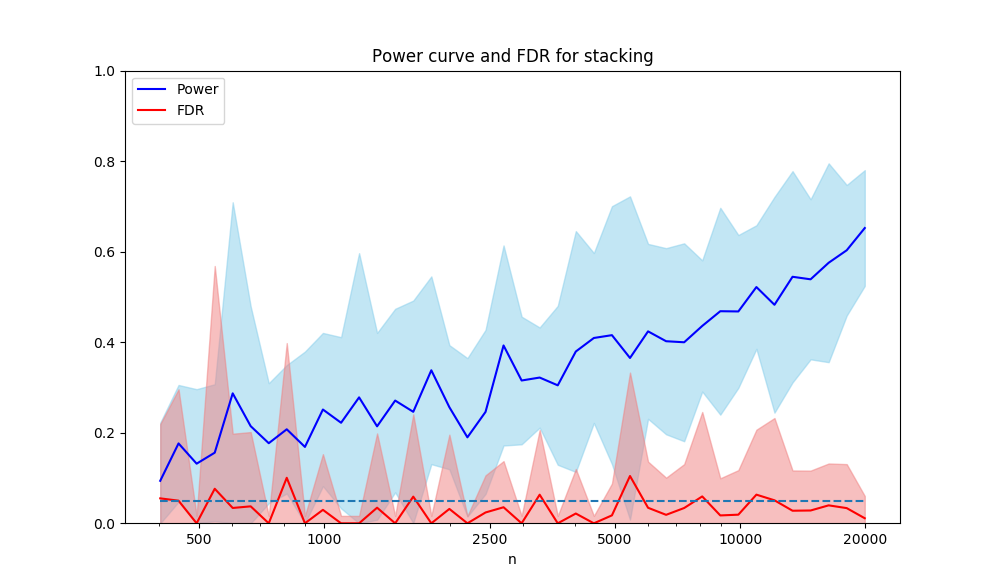}
    \end{subfigure}
    ~
    \begin{subfigure}[t]{\textwidth}
        \centering
        \includegraphics[width = 0.7\textwidth]{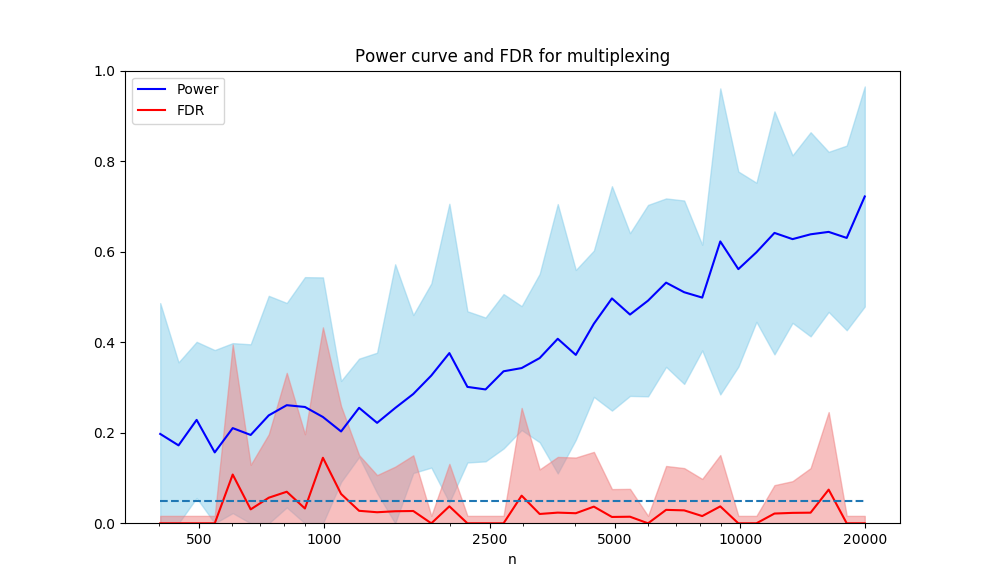}
    \end{subfigure}
    ~
    \begin{subfigure}[t]{\textwidth}
        \centering
        \includegraphics[width = 0.7\textwidth]{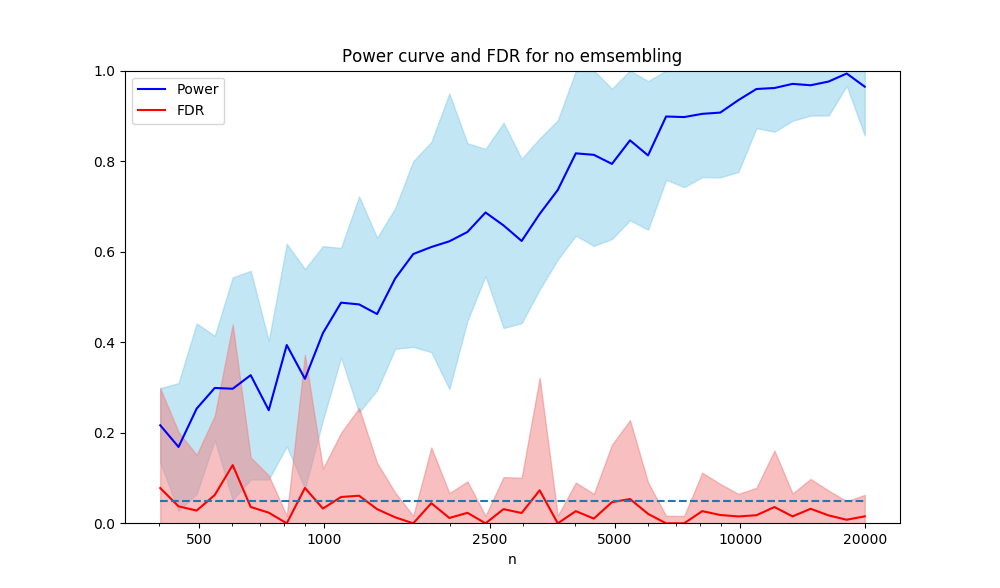}
    \end{subfigure}
\caption[Performance graphs]{Power (blue lines) and FDR (red lines) for increasing sample size for all 3 methods, showing that the performance increases as expected for all 3 methods, when the sample size is increased. The transparent blue and red areas denote the 90\% confidence intervals, the dashed line shows the expected FDR (0.05)}
\label{fig: performance}
\end{figure}

\begin{table}
\centering
\begin{tabular}{l | l l}
 & FDR & Time (sec)\\ \hline
No ensembling & 3.09\% & 30\\
Stacking & 3.03\% & 450 \\
Multiplexing & 2.75\% & 1000
\end{tabular}
\caption[Performance table]{False-discovery rates and run times for a data set of 22000 for all used methods}
\label{tab:diagnost}
\end{table}

Table \ref{tab:diagnost} shows the average FDR and run times for each of the three methods. The average FDR seems to be similar across all three methods, whereas the computational complexities differ by a large amount. No ensembling PCIT runs very quick, about 15 times faster than stacking, which itself only takes about half as long as multiplexing. This is the case, since multiplexing requires the calculation of performance measures for each used estimator. Figure \ref{fig: performance} shows the power and FDR of the algorithm for increasing sample size. The FDR for all 3 methods seem to be slightly higher for small sample sizes than they are for large sample sizes, but they are generally around or below the desired 5\% in the expectation (the variances are quite high, as the number of possible reruns is low due to the computational complexity of the multiplexing method). While it might seem surprising that stacking and multiplexing are outperformed by the no-ensembling method, one has to remember that the ensembling is used to choose the optimal prediction functional automatically from a set of estimators. However, the data is multivariate Gaussian, for which ridge regression is optimal in terms of generalization performance. While stacking and multiplexing are tasked to find this property in a large set of estimators, the estimator used in the no ensembling case is Elastic Net regularized linear regression, a generalization of ridge regression, and hence fares better since there is less variance in finding the optimal estimator. For all three methods, the power increases roughly logarithmically in the sample size, implying that for a test that recovers the truth with high probability, a large data set or a more powerful set of estimators (see Section \ref{sec:econ}) might be needed for that specific task. However, asymptotically, all three tests seem to recover a graph that is close to the truth, in this specific scenario, unless the power starts to plateau before reaching 1 (which there is no indication for). Since the power for the no ensembling case is biased by the fact that it uses an optimal prediction functional, the power curves for stacking and ensembling provide a better estimate for the performance on an unseen data set of an unknown distribution family. As graphical model structure learning is an inherently hard problem (due to issues such as multiple testing, combinatorial explosion, as outlined in Sections \ref{sec:gm}), it is promising that the algorithm finds an increasing fraction of the links while keeping the ratio of false-discoveries constant.

\subsubsection{Performance of structure learning algorithm: Variance of estimated model}
As a second performance metric, this section will assess the consistency of the learned structures on resamples from a data set $\mathcal D$. Assuming that all the observation in $\mathcal D$ are identically distributed, the structure learning method should arrive at the same conclusions on the resamples, less some variance in the process. The Auto MPG Data Set\footnote{\url{https://archive.ics.uci.edu/ml/datasets/auto+mpg}} containing various continuous and discrete car performance-attributes, from the UCI machine learning repository \cite{lichman2013}, is used to conduct the test. The data set contains 398 instances of 8 variables. For the purpose of the experiment, data sets of the same size will be sampled with replacement from the full data set 100 times. On each resample, a graph is learned using the stacking estimator on a 10\% FDR level. Each subsample contains about two thirds of the original data points (with some instances repeated). This is a commonly used procedure to estimate the sampling distribution of a statistics, and which will here allow us to assess the variance in the learned graph structure. Figure \ref{fig: performresamp} shows the results. On average, there are 6 links in the learned structure, hence the FDR advocates about 0.6 type 1 errors per learned model. The green lines are connections that are found less times than expected by the FDR and the blue lines are connections that are found in a large fraction of the models (over one third of the resamples). The concerning links are the ones in between, for which the null of independence is rejected more than occasionally, but not in a reasonably large fraction of the learned graphs. There are only 2 links in the model for which this occurs, so overall, the variance across learned graphs seems to be reasonably small and the learning process is hence consistent.

\begin{figure}
\centering
\includegraphics[width = \textwidth]{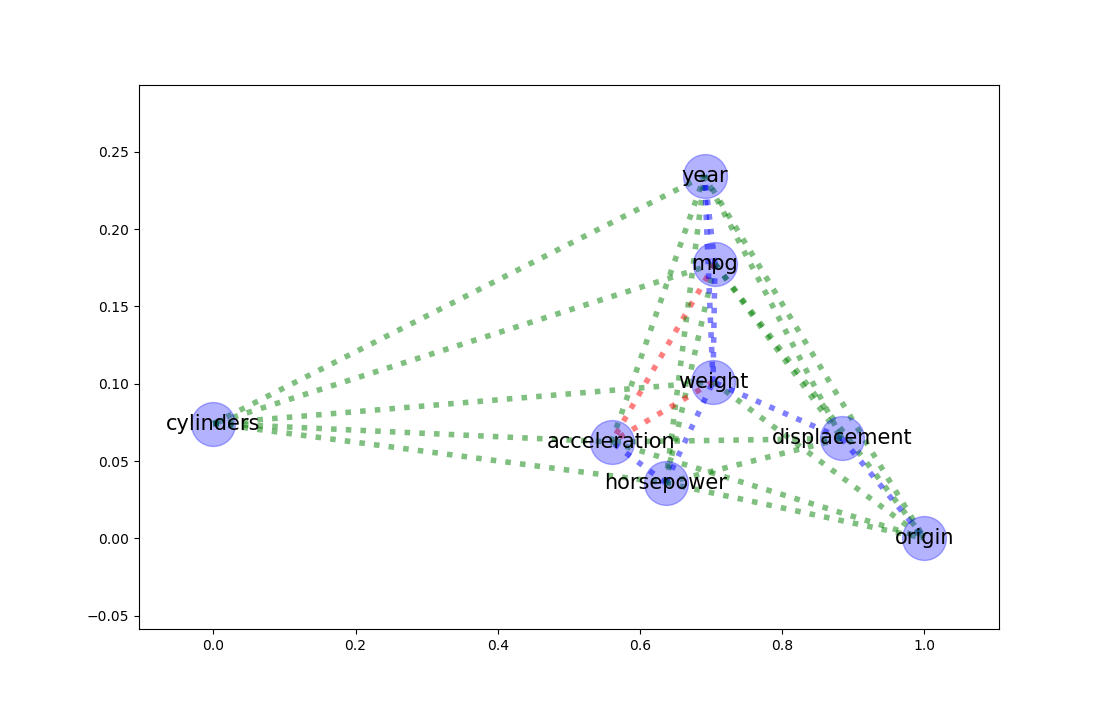}
\caption[Consistency under resamples]{Frequencies of edge occurrence in the learned structure. Green denotes edges that occur in less than 7\% (as advocated by FDR) of the models, blue for edges that occur in more than a third of the model, and red everything in between}
\label{fig: performresamp}
\end{figure}

\subsection{Experiments on real data sets}
In this section, the outputs of the graphical model structure learning routine are shown on a selection of real world data sets. It will outline some possibilities of the user to trade off between power and computational complexity.
\subsubsection{Sklearn data sets: Boston Housing and Iris}\label{sec:sklearndata}
\textbf{Setup}: One receives a data set and is interested in initial data exploration. This involves finding related sets of variables, and variables that lack obvious relationships with other variables.\\[0.1in]
Boston Housing and Iris are two standard data sets from sklearn. The housing data set contains 506 samples of 14 variables, originally collected to build a model for predicting house prices (descriptions of the variables can be found online\footnote{\url{http://scikit-learn.org/stable/modules/classes.html\#module-sklearn.datasets}}). The Iris data set is a data set for predicting flower species based on petal and sepal measurements with 150 observations in 5 dimensions.

\begin{figure*}
    \centering
        \begin{subfigure}[t]{0.5\textwidth}
        \centering
        \includegraphics[width = \textwidth, height = 1.8in]{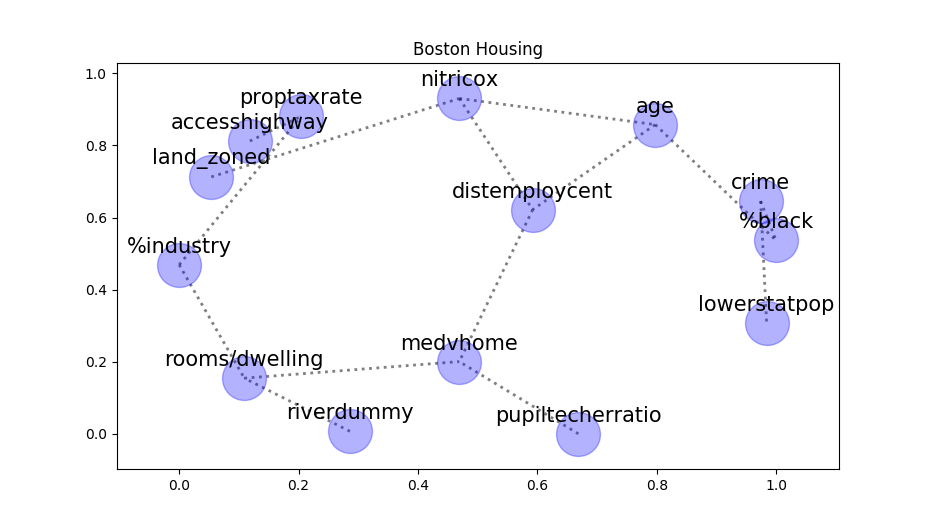}
    \end{subfigure}%
    ~
    \begin{subfigure}[t]{0.5\textwidth}
        \centering
        \includegraphics[width = \textwidth, height = 1.8in]{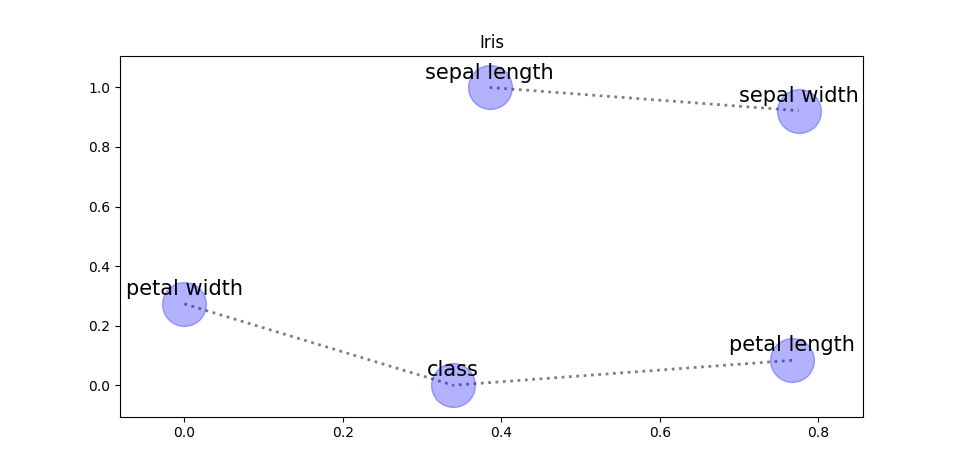}
    \end{subfigure}%
    \caption[Graphs for Sklearn data sets]{Learned graphical model structures for the Boston Housing and the Iris data set}
    \label{fig: sklearndata}
\end{figure*}
Estimation will take place using the default stacking estimator. Since we are interested in initial exploration, and finding interesting groups of variables, a large FDR (20\%) was chosen. Note that, unlike for other (mostly score-based) structure learning algorithms, the outcome of one experiment (the presence of an edge in the model) does not influence other experiments, and hence, false discoveries do not compromise the estimated structure additionally.\\

The results are shown in Figure \ref{fig: sklearndata} (graphs drawn with NetworkX \citep{networkx2017}). For the Boston housing data, seemingly sensible variable groupings occur. In the top left, the variables related to industrialization of a neighbourhood are shown, while on the right, demographic attributes form a cloud of related variables. For the Iris data set, while length and width of both sepal and petal are related, as expected, it seems that petal size has a higher association with class, and, in fact, width and length are independent given the class.\\
Both of these analyses require no parameter tuning (the only non-default chosen for this experiment was the adjusted confidence level) and take very little time (less than 15 seconds). The implementation of algorithm \ref{alg:undirstruct} thus provides a quick, ready to use tool for initial exploratory data analysis.

\subsubsection{Key short-term economic indicators (UK)}\label{sec:econ}
\textbf{Setup}: One is interested in finding the relationships within a set of variables and making local conditional independence statements for sets of variables. The focus is on finding associations that we are confident about.\\[0.1in]
The economic indicator data set contains monthly data of key economic indicators between February 1987 and June 2017 for the UK from the OECD statistics database \citep{oecd2017}. The economics indicators are export and import figures, stock price and industrial production levels, overnight and 3 month Inter-bank interest rates, money supply (M4) and GDP. This rather small data set, around 369 instances of 9 variables, will outline the possibility of a user to trade off between computational complexity and power of the learning algorithm. As for most economic data sets, the signal to noise ratio can be quite low. Figure \ref{fig: econdefault} shows the structure learned by the default approach for confidence (false discovery-rate) level of 5\%, with a run time of 15 seconds. While the ground truth is not known (economic theory aside), it is apparent that many links are missed by the fast default approach. This becomes evident when considering a variable like the GDP, which (by definition) includes exports and imports, and hence some connection between the variables should exist. If there is need for a more powerful approach, a more powerful estimator can be chosen. Figure \ref{fig: econsvm} show the learned structure for an estimator that resamples the data 50 times and learns 50 different Support Vector Machines with automatically tuned hyperparameters. The implementation of this is straightforward, since sklearn provides hyperparameter and bagging wrappers, so an estimator can be tuned easily and passed to the MetaEstimator. The graph shows the many edges that were found. While it is impossible to judge the correctness of the graph, it seems that some reasonable groups of variables are found, such as industrial production, exports and imports (in economic theory, countries seek to balance their trade deficit), or the grouping of stock prices and the main drivers of inflation, money supply and interest rates. Additionally, all the variables are connected in some way, which would be expected from a data set of economic indicators. The computational complexity of this approach was rather high, with a run time of about 20 minutes, however this shows how easily the user can trade off between power and complexity,without having a large effect on the false discovery-rate (as shown in Section \ref{sec:performance}).
\begin{figure*}
    \centering
    \begin{subfigure}[t]{0.5\textwidth}
        \centering
        \includegraphics[width = \textwidth, height = 1.8in]{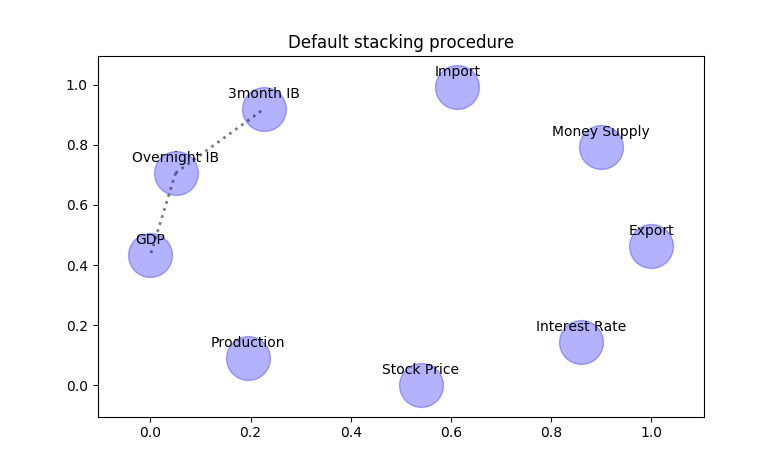}
        \caption{Default Approach}
        \label{fig: econdefault}
    \end{subfigure}%
    ~
    \begin{subfigure}[t]{0.5\textwidth}
        \centering
        \includegraphics[width = \textwidth, height = 1.8in]{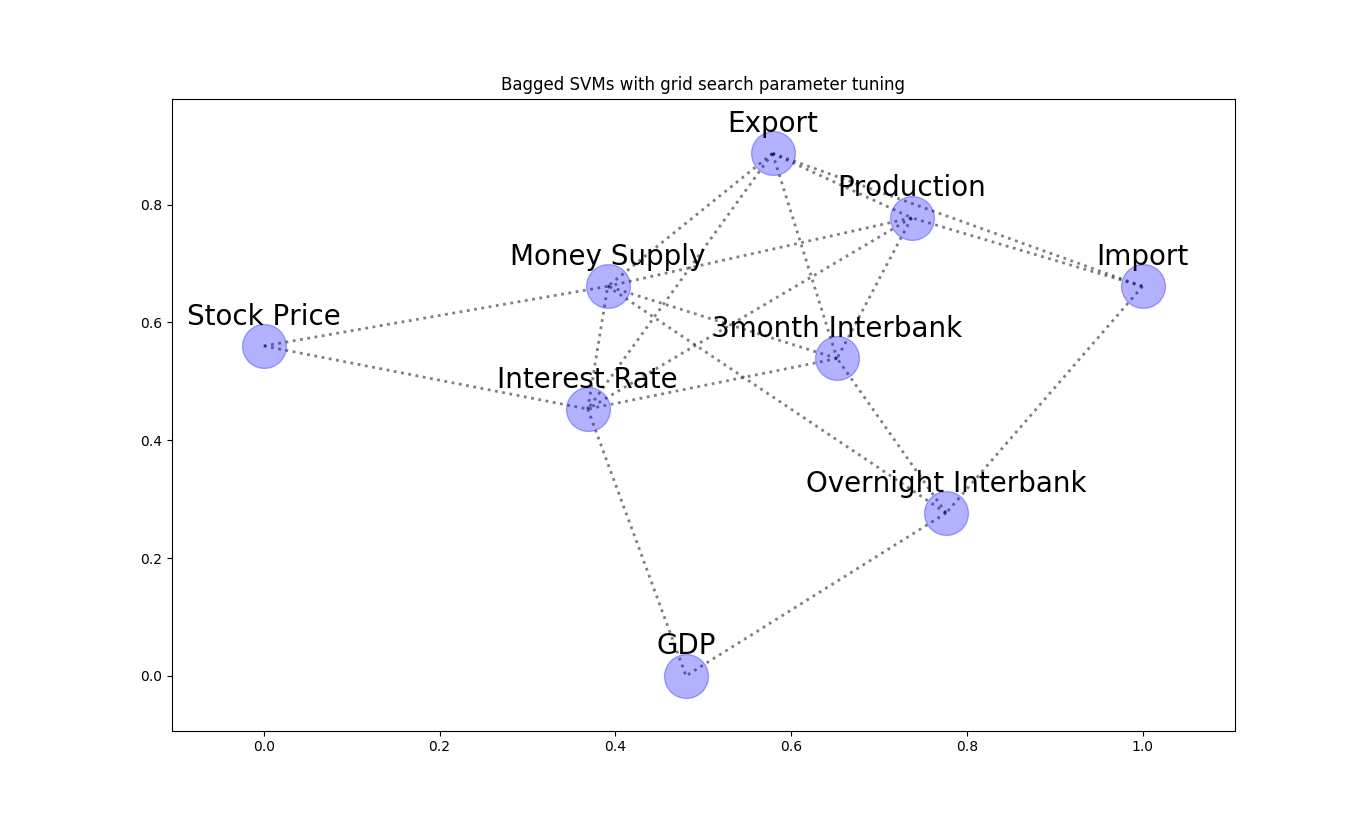}
        \caption{Bagged Support Vector Machines}
        \label{fig: econsvm}
    \end{subfigure}
    \caption[Economic Indicator graphs]{Learned graphs for the Economic Indicator data set for the default approach (left) and a more powerful approach using bagged SVM's (right)}
\end{figure*} 

%% file: 8_conclusions.tex
\section{Discussion}\label{sec:conclusion}
This paper has introduced a novel way for multivariate conditional independence testing based on predictive inference and linked to a supervised learning workflow, which addresses many of the current issues in independence testing:
\begin{itemize}
    \item \textbf{Few subjective choices}: By connecting the classical task of independence testing to supervised learning, and its well-understood hyperparameter tuning capabilities, there is no need for the heuristics and manual tuning choices prevalent in current state-of-the-art conditional independence tests
    \item \textbf{Low computational complexity}: By linking independence testing to one of the most-researched fields in data science, predictive modelling, the high level of power in the state-of-the-art-prediction methods can be directly benefited from, and any efficiency gains in the latter directly benefits the former.
    \item \textbf{High-dimensional problems}: By connecting the problem to the task of supervised learning, the method easily deals with the multivariate case, largely removing any issues concerning dimensionality of the data
\end{itemize}
It is important to note, that some choices remain necessary, as is the case in statistical methodology in general. Since it is not possible to test for an infinite amount of loss functions and predictive models, one has to decide on a subset to conduct this test. The larger the subset, the higher the data requirements to arrive at a respective power level. How the outlined methodology differs from the methods reviewed in section \ref{sec:background} is by outlining a principled way to choose from a set of subjective choices, and by using a subset of all-star predictive modelling algorithms to ensemble over as a default, a test that is expected to be able to deal with most of the usual scenarios, given a reasonable sample size.\\[0.1in]
To validate these claims, the test was bench-marked against current state-of-the-art conditional independence tests, and showed a similar or better performance in regions where the power of the tests exceeded 10\%.\\
Subsequently, the PCIT was applied in a method for learning the structure of an undirected graph best describing the independence structure of an input data set. The combination of the new conditional independence test and the structure learning algorithm address some of the current issues in graphical model structure learning:
\begin{itemize}
    \item \textbf{Low computational complexity}: While current exact algorithms, such as the PC-algorithm, often require a number of tests that is exponential in the number of variables, the proposed algorithm runs in quadratic time in the size of the graph. Additionally, a straightforward power-complexity trade off is provided
    \item \textbf{Exact algorithm}: Unlike many scored-based methods, the algorithm does not make any strong constraints on the underlying distribution and is not subject to local optima
    \item \textbf{False-discovery rate control}: FDR control is added to provide a powerful tool to control the fraction of type 1 errors when the number of variables is increasing
\end{itemize}
Performance tests showed that the false-discovery rate is as advertised and the power of the test increases constantly in the number of samples. Additionally, consistency under perturbations in the data set was demonstrated.\\
The algorithms have been distributed in the pcit package, providing users with an easy-to-use implementation to test for multivariate and conditional independence, as well as to perform graphical model structure learning with little need for manual parameter tuning. The implementations are particularly interesting for users 
\begin{itemize}
    \item assessing the value of collecting additional variables for a prediction task
    \item in need of a conditional independence test for multivariate data
    \item performing exploratory data analysis
    \item looking for a visual way to communicate the structure in their data set
\end{itemize}

There are a few ways in which the current implementation can be generalized to make for a more powerful test. The power of the conditional independence test is directly linked to the power of the supervised learning algorithm to find the optimal prediction functional and the correct loss. Currently, only the necessary minimum of 2 loss functions is implemented, one for regression tasks and one for classification tasks, but this can easily be generalized by including more losses, and checking if the baseline can be beaten with statistical significance in any of them. This would also strengthen the argument when reasoning about variables being independent, when the null hypothesis cannot be rejected. What's more, the current methodology connected single-output prediction with FDR control to make statements about the behaviour between the joint distributions. While this procedure results in a logically sound test, the feasibility of multi-output predictors, predicting several outputs at once, should be assessed, for appropriate multivariate loss functions.\\
On the other side, some extensions of the tests need to be conducted before a definitive conclusion can be made as to its power. In terms of performance evaluation, the power of the proposed routine was assessed for two specific cases only, for multivariate Gaussian data, and a simple conditional independence test. To get a better idea of the general feasibility of the algorithm, more scenarios need to be analyzed and performance tests conducted. Additionally, the power of the test in the context of alternative graphical model structure learning algorithms should be evaluated.\\

%% file: 9_appendix.tex
\addcontentsline{toc}{section}{Appendices}

\appendix
\section{Best uninformed predictors: classification}
This appendix collects proofs of number of elementary computations to obtain some of the explicit statements about classification losses found in the main corpus.

\subsection{Misclassification loss is a probabilistic loss}\label{app:classifdet}
In this sub-section, we consider the alternative misclassification loss
$$L:\calY'\times\calY\rightarrow [0,1];\quad (p,y)\mapsto 1-p(y),$$
with $\calY$ being a discrete set, and $\calY'$ being the probability simplex of probability mass functions on $\calY$, as considered in Remark~\ref{Rem:class}.
In said remark, it is claimed that rounding a probabilistic prediction to a deterministic one never increases the generalization loss. We first prove this for the unconditional case which is equivalent to constant predictions:

\begin{Lem}\label{Lem:classifdet}
Let $Y$ be an $\calY$-valued random variable with probability mass function $p_Y$.
There is always a deterministic minimizer of the expected generalization loss, i.e., there exists a pmf $p_0:\calY\rightarrow \{0,1\}$ such that
$$p_0 = \underset{p\in\calY'}{\argmin}\; \EE[L(p,Y)].$$
\end{Lem}
\begin{proof}
Let $p:\calY\rightarrow [0,1]$ be any pmf. For its expected generalization loss, one has, by substituting definitions,
$$\EE[L(p,Y)] = \sum_{y\in \calY} p(y)(1-p_Y(y)) = 1 - \sum_{y\in \calY} p(y) p_Y(y).$$
Let $y_0:= \underset{y\in\calY}{\argmax}\; p_Y(y)$ (if there are multiple maximizers, choose arbitrarily). By construction, it holds that
$p_Y(y_0)\ge \sum_{y\in \calY} p(y) p_Y(y)$.
Thus, for $p_0: y\mapsto [1\;\mbox{if}\;y=y_0\;;0\;\mbox{otherwise}]$, one has
$$\EE[L(p_0,Y)] = 1 - p_Y(y_0) \le 1 - \sum_{y\in \calY} p(y) p_Y(y) = \EE[L(p,Y)].$$
Since $p$ was arbitrary, this proves the claim.
\end{proof}

Note that Lemma~\ref{Lem:classifdet} does not exclude that there are other minimizers of the expected generalization loss (in general there will be an infinity), it only states that a deterministic minimizer may be found.

\begin{Prop}
Let $X,Y,Z$ be a random variables taking values in $\calX,\calY,\calZ$.
Then one can make choices which always predict a deterministic class, for the following prediction functionals as considered in Section~\ref{sec:depend}:
\begin{itemize}
\item[(i)] the best uninformed predictor $\omega^{(L)}_{Y}:\calX\rightarrow \calY'$
\item[(ii)] the best predictor $\omega^{(L)}_{Y|X}:\calX\rightarrow \calY'$
\item[(iii)] the best conditionally uninformed predictor $\omega^{(L)}_{Y|Z}:\calX\times \calZ\rightarrow \calY'$
\item[(iv)] the best conditional predictor $\omega^{(L)}_{Y|X,Z}:\calX\times \calZ\rightarrow \calY'$
\end{itemize}
That is, in all four cases, choices may be made where the image is always a deterministic pmf $p:\calY\rightarrow \{0,1\}$.
\end{Prop}
\begin{proof}
(i) follows directly from Lemma~\ref{Lem:classifdet} which directly implies that the function $y\mapsto p_0$ is the best constant predictor and hence the best uninformed predictor, where $p_0$ is defined as in the statement (or more constructively in the proof) of Lemma~\ref{Lem:classifdet}, thus $\omega^{(L)}_{Y}:x\mapsto p_0$ is a possible choice.\\
(ii) follows from noting that the Lemma~\ref{Lem:classifdet} and its proof remain still valid when considering the conditional under $X$, i.e., defining a conditional
$$p_0: \calX\rightarrow \calY';\; x\mapsto \underset{p\in\calY'}{\argmin} \EE[L(p,Y)|X=x]$$
and thus $\omega^{(L)}_{Y|X}: x\mapsto p_0(x)$.\\
(iii) and (iv) follow analogously by additional conditioning on $Z$ in the same way.
\end{proof}

\subsection{Logarithmic loss is a proper loss}\label{app:logloss}
In this sub-section, we consider the logarithmic loss (or cross-entropy loss)
$$L:\calY'\times\calY\rightarrow \RR^+;\quad (p,y)\mapsto -\log p(y),$$
with $\calY$ being a discrete set, and $\calY'$ being the probability simplex of probability mass functions on $\calY$, as considered in Remark~\ref{Rem:class}.

\begin{Prop}
The expected generalization log-loss is minimized by the true distribution. I.e., let $Y$ be random variable taking values in $\calY$, with probability mass function $p_Y$. Then,
$$p_Y = \underset{y\in\calY'}{\argmin}\;\EE[L(y,Y)].$$
\end{Prop}
\begin{proof}
Let $p\in \calY$ be arbitrary.
Substituting definitions, it holds that
\begin{align*}
\EE[L(p,Y)] &= -\sum_{y\in \calY} p_Y(y)\log p(y)\\
&\ge -\sum_{y\in \calY} p(y)\log p(y)\\
& = \EE[L(p_Y,Y)],
\end{align*}
where the inequality in the middle is Gibbs' inequality.
\end{proof}

\subsection{Brier loss is a proper loss}\label{app:brier}
In this sub-section, we consider the Brier loss (or squared classification loss)
$$L:\calY'\times\calY\rightarrow \RR^+;\quad (p,y)\mapsto (1-p(y))^2 + \sum_{y'\neq y}p(y')^2,$$
with $\calY$ being a discrete set, and $\calY'$ being the probability simplex of probability mass functions on $\calY$, as considered in Remark~\ref{Rem:class}.

\begin{Prop}
The expected Brier loss is minimized by the true distribution. I.e., let $Y$ be random variable taking values in $\calY$, with probability mass function $p_Y$. Then,
$$p_Y = \underset{y\in\calY'}{\argmin}\;\EE[L(y,Y)].$$
\end{Prop}
\begin{proof}
Let $p\in \calY$ be arbitrary.
By applying the binomial rule, observe that
$$L(p,y) = 1 - 2p(y) + \sum_{y'\in \calY}p(y')^2$$
Substituting definitions, it holds that
\begin{align*}
\EE[L(p,Y)] &= \sum_{y\in\calY}\left( p_Y(y) - 2p_Y(y)p(y) + p_Y(y)\sum_{y'\in \calY}p(y')^2\right)\\
&= 1 - 2 \sum_{y\in \calY} p_Y(y)p(y) + \sum_{y\in \calY}p(y)^2\\
&= 1 - \sum_{y\in \calY}p_Y(y)^2 + \sum_{y\in \calY}p_Y(y)^2 - 2 \sum_{y\in \calY} p_Y(y)p(y) + \sum_{y\in \calY}p(y)^2\\
& = 1 - \sum_{y\in \calY}p_Y(y)^2 + \sum_{y\in \calY}(p(y) - p_Y(y))^2.
\end{align*}
Note that the first two terms (the $1$ and the sum over $p_Y(y)$) do not depend on $p$, while the last term is non-negative and minimized with a value of zero if and only if $p=p_Y$.\\
Since $p$ was arbitrary, this proves the claim.
\end{proof}

\section{Elicited statistics for regression losses}
This appendix collects explicit proofs of the elicited statistics for squared loss, and Q-loss (and the absolute loss which is a special case).

\subsection{Squared loss elicits the mean}\label{app:sqloss}
In this sub-section, we consider the univariate regression case $\calY = \RR$, and the squared loss
$$L:\calY\times\calY\rightarrow \RR;\; (y, y^*) \mapsto (y - y^*)^2.$$

\begin{Prop}
The squared loss elicits the mean. I.e., let $Y$ be random variable taking values in $\calY$. Then,
$$\EE[Y] = \underset{y\in\calY}{\argmin}\;\EE[L(y,Y)].$$
\end{Prop}
\begin{proof}
Substituting definitions, it holds that
\begin{align*}
\EE[L(y,Y)] &= \EE[(y-Y)^2]\\
&= \EE[y^2-2yY + Y^2]\\
&= y^2-2y\EE[Y] + \EE[Y^2]\\
&= y^2-2y\EE[Y] + \EE[Y^2] - \EE[Y^2] + \EE[Y^2]\\
& =(\EE[Y] - y)^2 + \Var(Y),
\end{align*}
which is the well-known derivation of the bias-variance decomposition.\\
The first term, $(\EE[Y] - y)^2$, is minimized whenever $\EE[Y] = y$, and the second term, $\Var(Y)$, does not depend on $y$. Thus, the sum of both terms is minimized (in $y$ while fixing $Y$) for the choice $y=\EE[Y]$.
\end{proof}

\subsection{Quantile loss elicits the quantile}\label{app:Qloss}
In this sub-section, we consider the univariate regression case $\calY = \RR$, and the Q-loss (or quantile loss)
$$L:\calY\times\calY\rightarrow \RR;\;(y,y_*)\mapsto \alpha\cdot m(y_*,y) + (1-\alpha)\cdot m(y,y_*),$$
with $m(x,z)=\min(x-z,0).$

\begin{Prop}
The Q-loss elicits the $\alpha$-quantile. I.e., let $Y$ be random variable taking values in $\calY$ with cdf $F: \RR \rightarrow [0,1]$. Then,
$$F^{-1}_Y(\alpha) = \underset{y\in\calY}{\argmin}\;\EE[L(y,Y)].$$
\end{Prop}
\begin{proof}
We first assume that $Y$ is absolutely continuous, i.e., $Y$ has a probability density function $p:\RR\rightarrow \RR^+$ and $F$ is bijective.
%Let $F_{-Y}$ be the cdf of $-Y$.
One then computes
\begin{align*}
\mathbb{E}_Y[L_\alpha(y^*,Y)] & = \int_{-\infty}^{y^*} (1 - \alpha)(y^* - y)p(y)dy - \int_{y^*}^\infty (\alpha)(y^* - y)p(y)dy
\\ & = y^*(F(y^*) - \alpha) + \alpha \mathbb{E}[Y] - \int_{-\infty}^{y^*}{ y p(y) dy}
\\ \frac{\partial \mathbb{E}_Y[L_\alpha(y^*,Y)]}{\partial y^*} & = y^* p(y^*) + F(y^*) - \alpha - y^* p(y^*) = F(y^*) - \alpha \stackrel{!}{=} 0
\\ & \implies P(y^*) = \alpha
\\ \frac{\partial ^2 \mathbb{E}_Y[L_\alpha(y^*,Y)]}{\partial (y^*)^2} & = p(y^*) \geq 0
\end{align*}
Hence, the first order condition is a minimum, minimized by the $\alpha$-quantile of $Y$, and thus the quantile loss elicits the quantile.\\
For general $Y$, note that $F$ always exists, and thus when $p$ appears inside integrals, the integrals well-defined. The partial derivatives may not always be defined but is the same as the sign of sufficiently small finite differences, thus the proof logic follows through for the general case. In case of jump discontinuities of $F$, any monotone inverse $F^{-1}$ may be chosen for the statement.
\end{proof}

%% file: pcit.bbl
\begin{thebibliography}{43}
\providecommand{\natexlab}[1]{#1}
\providecommand{\url}[1]{\texttt{#1}}
\expandafter\ifx\csname urlstyle\endcsname\relax
  \providecommand{\doi}[1]{doi: #1}\else
  \providecommand{\doi}{doi: \begingroup \urlstyle{rm}\Url}\fi

\bibitem[Aggarwal(2014)]{Aggarwal2014}
C.C. Aggarwal.
\newblock \emph{Data Classification: Algorithms and Applications}.
\newblock Chapman \& Hall/CRC, 1st edition, 2014.

\bibitem[Barber(2012)]{barber2012}
D.~Barber.
\newblock \emph{Bayesian reasoning and machine learning}.
\newblock Cambridge University Press, 2012.

\bibitem[Baringhaus and Franz(2004)]{Baringhaus2004}
L.~Baringhaus and C.~Franz.
\newblock On a new multivariate two-sample test.
\newblock \emph{Journal of Multivariate Analysis}, 88:\penalty0 190--206, 2004.

\bibitem[Belloni et~al.(2017)Belloni, Chen, Chernozhukov,
  et~al.]{belloni2017quantile}
Alexandre Belloni, Mingli Chen, Victor Chernozhukov, et~al.
\newblock Quantile graphical models: prediction and conditional independence
  with applications to systemic risk.
\newblock Technical report, Centre for Microdata Methods and Practice,
  Institute for Fiscal Studies, 2017.

\bibitem[Benjamini and Hochberg(1995)]{benjamini1995}
Y.~Benjamini and Y.~Hochberg.
\newblock Controlling the false discovery rate: A practical and powerful
  approach to multiple testing.
\newblock \emph{Journal of the Royal Statistical Society. Series B
  (Methodological)}, 57:\penalty0 289--300, 1995.

\bibitem[Benjamini and Yekutieli(2001)]{Benjamini2001}
Y.~Benjamini and D.~Yekutieli.
\newblock The control of the false discovery rate in multiple testing under
  dependency.
\newblock \emph{The Annals of Statistics}, 29:\penalty0 1165--1188, 2001.

\bibitem[{Bergsma}(2011)]{bergsma2011}
W.~{Bergsma}.
\newblock {Nonparametric testing of conditional independence by means of the
  partial copula}.
\newblock \emph{ArXiv e-prints: 1101.4607}, 2011.

\bibitem[Bergsma(2004)]{bergsma2004}
W.P. Bergsma.
\newblock \emph{Testing conditional independence for continuous random
  variables}.
\newblock Eurandom, 2004.

\bibitem[Buitinck et~al.(2013)Buitinck, Louppe, Blondel, Pedregosa, Mueller,
  Grisel, Niculae, Prettenhofer, Gramfort, Grobler, et~al.]{buitinck2013}
L.~Buitinck, G.~Louppe, M.~Blondel, F.~Pedregosa, A.~Mueller, O.~Grisel,
  V.~Niculae, P.~Prettenhofer, A.~Gramfort, J.~Grobler, et~al.
\newblock {API} design for machine learning software: experiences from the
  scikit-learn project.
\newblock \emph{arXiv preprint arXiv:1309.0238}, 2013.

\bibitem[Cherubini et~al.(2004)Cherubini, Luciano, and
  Vecchiato]{cherubini2004}
U.~Cherubini, E.~Luciano, and W.~Vecchiato.
\newblock \emph{Copula methods in finance}.
\newblock John Wiley \& Sons, 2004.

\bibitem[Chwialkowski et~al.(2015)Chwialkowski, Ramdas, Sejdinovic, and
  Gretton]{Chwialkowski2015}
K.P. Chwialkowski, A.~Ramdas, D.~Sejdinovic, and A.~Gretton.
\newblock Fast two-sample testing with analytic representations of probability
  measures.
\newblock In \emph{Advances in Neural Information Processing Systems 28}, pages
  1981--1989. Nips, 2015.

\bibitem[Corder and Foreman(2009)]{Corder2009}
G.W. Corder and D.I. Foreman.
\newblock \emph{Nonparametric Statistics for Non-Statisticians: A Step-by-Step
  Approach}.
\newblock John Wiley \& Sons, Inc., 2009.

\bibitem[Dionisio et~al.(2006)Dionisio, Menezes, and Mendes]{Dionisio2006}
A.~Dionisio, R~Menezes, and D.A. Mendes.
\newblock Entropy-based independence test.
\newblock \emph{Nonlinear Dynamics}, 44:\penalty0 351--357, 2006.

\bibitem[Fern{\'a}ndez et~al.(2008)Fern{\'a}ndez, Gamero, and
  Garc{\'i}a]{Alba2008}
V.A. Fern{\'a}ndez, M.D.J. Gamero, and J.M. Garc{\'i}a.
\newblock A test for the two-sample problem based on empirical characteristic
  functions.
\newblock \emph{Computational Statistics \& Data Analysis}, 52:\penalty0
  3730--3748, 2008.

\bibitem[Fragoso and Neto(2015)]{Fragoso2015}
T.M. Fragoso and F.L. Neto.
\newblock Bayesian model averaging: A systematic review and conceptual
  classification.
\newblock \emph{arXiv preprint arXiv:1509.08864}, 2015.

\bibitem[Friedman et~al.(2001)Friedman, Hastie, and Tibshirani]{friedman2001}
J.~Friedman, T.~Hastie, and R.~Tibshirani.
\newblock \emph{The elements of statistical learning}, volume~1.
\newblock Springer series in statistics New York, 2001.

\bibitem[Genest and R{\'e}millard(2004)]{genest2004}
C.~Genest and B.~R{\'e}millard.
\newblock Test of independence and randomness based on the empirical copula
  process.
\newblock \emph{Test}, 13:\penalty0 335--369, 2004.

\bibitem[Gneiting and Raftery(2007)]{gneiting2007strictly}
T.~Gneiting and A.E. Raftery.
\newblock Strictly proper scoring rules, prediction, and estimation.
\newblock \emph{Journal of the American Statistical Association}, 102:\penalty0
  359--378, 2007.

\bibitem[Gressmann et~al.(2018)Gressmann, Kir{\'a}ly, Mateen, and
  Oberhauser]{gressmann2018probabilistic}
Frithjof Gressmann, Franz~J Kir{\'a}ly, Bilal Mateen, and Harald Oberhauser.
\newblock Probabilistic supervised learning.
\newblock \emph{arXiv preprint arXiv:1801.00753}, 2018.

\bibitem[Gretton and Gy{\u{A}}{\'s}rfi(2010)]{Gretton2010}
A.~Gretton and L.~Gy{\u{A}}{\'s}rfi.
\newblock Consistent nonparametric tests of independence.
\newblock \emph{Journal of Machine Learning Research}, 11:\penalty0 1391--1423,
  2010.

\bibitem[Gretton et~al.(2005)Gretton, Bousquet, Smola, and
  Scholkopf]{Gretton2005}
A.~Gretton, O.~Bousquet, A.~Smola, and B.~Scholkopf.
\newblock Measuring statistical dependence with hilbert-schmidt norms.
\newblock In \emph{ALT}, volume~16, pages 63--78. Springer, 2005.

\bibitem[Gretton et~al.(2008)Gretton, Fukumizu, Teo, Song, Sch{\"o}lkopf, and
  Smola]{Gretton2008}
A.~Gretton, K.~Fukumizu, C.H. Teo, L.~Song, B.~Sch{\"o}lkopf, and A.~Smola.
\newblock A kernel statistical test of independence.
\newblock In \emph{Advances in neural information processing systems}, pages
  585--592, 2008.

\bibitem[Gretton et~al.(2009)Gretton, Fukumizu, Harchaoui, and
  Sriperumbudur]{Gretton2009}
A.~Gretton, K.~Fukumizu, Z.~Harchaoui, and B.K. Sriperumbudur.
\newblock A fast, consistent kernel two-sample test.
\newblock \emph{Advances in Neural Information Processing Systems 22}, pages
  673--681, 2009.

\bibitem[Gretton et~al.(2012{\natexlab{a}})Gretton, Borgwardt, Rasch,
  Sch{\"o}lkopf, and Smola]{Gretton2012}
A~Gretton, K.M. Borgwardt, M.J. Rasch, B.~Sch{\"o}lkopf, and A.~Smola.
\newblock A kernel two-sample test.
\newblock \emph{Journal of Machine Learning Research}, 13:\penalty0 723--773,
  2012{\natexlab{a}}.

\bibitem[Gretton et~al.(2012{\natexlab{b}})Gretton, Sejdinovic, Strathmann,
  Balakrishnan, Pontil, Fukumizu, and Sriperumbudur]{GrettonNips2012}
A.~Gretton, D.~Sejdinovic, H.~Strathmann, S.~Balakrishnan, M.~Pontil,
  K.~Fukumizu, and B.K. Sriperumbudur.
\newblock Optimal kernel choice for large-scale two-sample tests.
\newblock In \emph{Advances in Neural Information Processing Systems 25}, pages
  1205--1213. Nips, 2012{\natexlab{b}}.

\bibitem[Guha and Mallick(2016)]{guha2016quantile}
Nilabja Guha and Bani~K Mallick.
\newblock Quantile graphical models: Bayesian approaches.
\newblock \emph{arXiv preprint arXiv:1611.02480}, 2016.

\bibitem[Hagberg et~al.(2008)Hagberg, Schult, and Swart]{networkx2017}
A.A. Hagberg, D.A. Schult, and P.J. Swart.
\newblock Exploring network structure, dynamics, and function using {NetworkX}.
\newblock In \emph{Proceedings of the 7th Python in Science Conference
  (SciPy2008)}, pages 11--15, 2008.

\bibitem[Jones et~al.(2001)Jones, Oliphant, and Peterson]{Scipy}
E.~Jones, T.~Oliphant, and P.~Peterson.
\newblock {SciPy}: Open source scientific tools for {Python}, 2001.
\newblock URL \url{http://www.scipy.org/}.

\bibitem[Koller and Friedman(2009)]{koller2009}
D.~Koller and N.~Friedman.
\newblock \emph{Probabilistic graphical models: principles and techniques}.
\newblock MIT press, 2009.

\bibitem[Lichman(2013)]{lichman2013}
M.~Lichman.
\newblock {UCI} machine learning repository, 2013.
\newblock URL \url{http://archive.ics.uci.edu/ml}.

\bibitem[Lopez-Paz and Oquab(2016{\natexlab{a}})]{Lopez2017}
D.~Lopez-Paz and M.~Oquab.
\newblock Revisiting classifier two-sample tests.
\newblock \emph{arXiv preprint arXiv:1610.06545}, 2016{\natexlab{a}}.

\bibitem[Lopez-Paz and Oquab(2016{\natexlab{b}})]{lopez2016revisiting}
D.~Lopez-Paz and M.~Oquab.
\newblock Revisiting classifier two-sample tests.
\newblock \emph{arXiv preprint arXiv:1610.06545}, 2016{\natexlab{b}}.

\bibitem[Nadeau(2003)]{nadeau2003}
Y.~Nadeau, C.and~Bengio.
\newblock Inference for the generalization error.
\newblock \emph{Machine Learning}, 52:\penalty0 239--281, 2003.

\bibitem[{OECD}(2017)]{oecd2017}
{OECD}.
\newblock {OECD} statistics: Key short-term economic indicators.
\newblock \url{http://stats.oecd.org/}, 2017.
\newblock Accessed: 2017-08-06.

\bibitem[Pedregosa et~al.(2011{\natexlab{a}})Pedregosa, Varoquaux, Gramfort,
  Michel, Thirion, Grisel, Blondel, Prettenhofer, Weiss, Dubourg, Vanderplas,
  Passos, Cournapeau, Brucher, Perrot, and Duchesnay]{scikit-learn}
F.~Pedregosa, G.~Varoquaux, A.~Gramfort, V.~Michel, B.~Thirion, O.~Grisel,
  M.~Blondel, P.~Prettenhofer, R.~Weiss, V.~Dubourg, J.~Vanderplas, A.~Passos,
  D.~Cournapeau, M.~Brucher, M.~Perrot, and E.~Duchesnay.
\newblock Scikit-learn: Machine learning in {P}ython.
\newblock \emph{Journal of Machine Learning Research}, 12:\penalty0 2825--2830,
  2011{\natexlab{a}}.

\bibitem[Pedregosa et~al.(2011{\natexlab{b}})Pedregosa, Varoquaux, Gramfort,
  Michel, Thirion, Grisel, Blondel, Prettenhofer, Weiss, Dubourg,
  et~al.]{pedregosa2011scikit}
Fabian Pedregosa, Ga{\"e}l Varoquaux, Alexandre Gramfort, Vincent Michel,
  Bertrand Thirion, Olivier Grisel, Mathieu Blondel, Peter Prettenhofer, Ron
  Weiss, Vincent Dubourg, et~al.
\newblock Scikit-learn: Machine learning in {P}ython.
\newblock \emph{Journal of Machine Learning Research}, 12\penalty0
  (Oct):\penalty0 2825--2830, 2011{\natexlab{b}}.

\bibitem[Raschka(2016)]{mlxtend}
S.~Raschka.
\newblock Mlxtend, 2016.
\newblock URL \url{http://dx.doi.org/10.5281/zenodo.594432}.

\bibitem[R{\'e}millard and Scaillet(2009)]{remillard2009}
B.~R{\'e}millard and O.~Scaillet.
\newblock Testing for equality between two copulas.
\newblock \emph{Journal of Multivariate Analysis}, 100:\penalty0 377--386,
  2009.

\bibitem[Schweizer and Wolff(1981)]{Schweizer1981}
B.~Schweizer and E.~F. Wolff.
\newblock On nonparametric measures of dependence for random variables.
\newblock \emph{The Annals of Statistics}, 9:\penalty0 879--885, 1981.

\bibitem[Sen et~al.(2017)Sen, Suresh, Shanmugam, Dimakis, and
  Shakkottai]{sen2017model}
Rajat Sen, Ananda~Theertha Suresh, Karthikeyan Shanmugam, Alexandros~G Dimakis,
  and Sanjay Shakkottai.
\newblock Model-powered conditional independence test.
\newblock In \emph{Advances in Neural Information Processing Systems}, pages
  2955--2965, 2017.

\bibitem[Sriperumbudur et~al.(2009)Sriperumbudur, Fukumizu, Gretton, Lanckriet,
  and Sch{\"o}lkopf]{sriperumbudur2009kernel}
Bharath~K Sriperumbudur, Kenji Fukumizu, Arthur Gretton, Gert~RG Lanckriet, and
  Bernhard Sch{\"o}lkopf.
\newblock Kernel choice and classifiability for {RKHS} embeddings of
  probability distributions.
\newblock In \emph{NIPS}, pages 1750--1758, 2009.

\bibitem[Zaremba et~al.(2013)Zaremba, Gretton, and Blaschko]{Zaremba2013}
W.~Zaremba, A.~Gretton, and M.~Blaschko.
\newblock B-test: A non-parametric, low variance kernel two-sample test.
\newblock \emph{Advances in Neural Information Processing Systems 26}, pages
  755--763, 2013.

\bibitem[Zhang et~al.(2012)Zhang, Peters, Janzing, and
  Sch{\"o}lkopf]{Zhang2012}
K.~Zhang, J.~Peters, D.~Janzing, and B.~Sch{\"o}lkopf.
\newblock Kernel-based conditional independence test and application in causal
  discovery.
\newblock \emph{arXiv preprint arXiv:1202.3775}, 2012.

\end{thebibliography}
